\documentclass[sn-basic,iicol]{sn-jnl-arxiv}

\usepackage{graphicx}
\usepackage{url}
\usepackage{bbm}
\usepackage{listings}
\usepackage{xcolor}
\usepackage{caption}
\usepackage{subcaption}
\usepackage[all]{nowidow}
\usepackage{amsthm}
\usepackage{enumitem}
\usepackage[capitalize]{cleveref}
\usepackage{etoolbox}

\crefname{section}{Sec.}{Secs.}
\Crefname{section}{Section}{Sections}
\Crefname{table}{Table}{Tables}
\crefname{table}{Tab.}{Tabs.}
\newtheorem{prop}{Proposition}
\newtheorem*{prop*}{Proposition}

\graphicspath{{images/}}

\definecolor{codegreen}{rgb}{0,0.6,0}
\definecolor{codegray}{rgb}{0.5,0.5,0.5}
\definecolor{codepurple}{rgb}{0.58,0,0.82}
\definecolor{backcolour}{rgb}{0.95,0.95,0.92}
\definecolor{light-gray}{gray}{0.9}
\definecolor{light-green}{rgb}{0.82, 0.94, 0.75}

\lstdefinestyle{mystyle}{
    backgroundcolor=\color{backcolour},   
    commentstyle=\color{codegreen},
    keywordstyle=\color{magenta},
    numberstyle=\tiny\color{codegray},
    stringstyle=\color{codepurple},
    basicstyle=\ttfamily\footnotesize,
    breakatwhitespace=false,         
    breaklines=true,                 
    captionpos=b,                    
    keepspaces=true,                 
    numbers=left,                    
    numbersep=5pt,                  
    showspaces=false,                
    showstringspaces=false,
    showtabs=false,                  
    tabsize=2,
    showlines=true
}

\jyear{2022}%

\theoremstyle{thmstyleone}%
\theoremstyle{thmstyletwo}%

\theoremstyle{thmstylethree}%

\begin{document}

\title[]{Similarity Contrastive Estimation for Image and Video Soft Contrastive Self-Supervised Learning}


\author*[1,2]{\fnm{Julien} \sur{Denize}}\email{julien.denize@cea.fr}

\author[1]{\fnm{Jaonary} \sur{Rabarisoa}}\email{jaonary.rabarisoa@cea.fr}
\equalcont{These authors contributed equally to this work.}

\author[1]{\fnm{Astrid} \sur{Orcesi}}\email{astrid.orcesi@cea.fr}
\equalcont{These authors contributed equally to this work.}

\author[2]{\fnm{Romain} \sur{Hérault}}\email{romain.herault@insa-rouen.fr}
\equalcont{These authors contributed equally to this work.}

\affil*[1]{\orgdiv{LVA}, \orgname{CEA LIST, Université Paris-Saclay}, \orgaddress{\city{Palaiseau}, \postcode{F-91120}, \country{France}}}

\affil*[2]{\orgdiv{LITIS}, \orgname{INSA Rouen, Normandie Univ}, \orgaddress{\city{Saint Etienne du Rouvray}, \postcode{76801}, \country{France}}}


\abstract{Contrastive representation learning has proven to be an effective self-supervised learning method for images and videos. Most successful approaches are based on Noise Contrastive Estimation (NCE) and use different views of an instance as positives that should be contrasted with other instances, called negatives, that are considered as noise. However, several instances in a dataset are drawn from the same distribution and share underlying semantic information. A good data representation should contain relations between the instances, or semantic similarity and dissimilarity, that contrastive learning harms by considering all negatives as noise. To circumvent this issue, we propose a novel formulation of contrastive learning using semantic similarity between instances called Similarity Contrastive Estimation (SCE). Our training objective is a soft contrastive one that brings the positives closer and estimates a continuous distribution to push or pull negative instances based on their learned similarities. We validate empirically our approach on both image and video representation learning. We show that SCE performs competitively with the state of the art on the ImageNet linear evaluation protocol for fewer pretraining epochs and that it generalizes to several downstream image tasks. We also show that SCE reaches state-of-the-art results for pretraining video representation and that the learned representation can generalize to video downstream tasks.}


\maketitle

\section{Introduction}
\label{sec:intro}

Self-Supervised learning (SSL) is an unsupervised learning procedure in which the data provides its own supervision to learn a practical representation of the data. A pretext task is designed to make this supervision. The pretrained model is then fine-tuned on downstream tasks and several works have shown that a self-supervised pretrained network can outperform its supervised counterpart for image \citep{Caron2020, Grill2020, Caron2021} and video \citep{Feichtenhofer2021, Duan2022}. It has been successfully applied to various image and video applications such as image classification, action classification, object detection and action localization.

Contrastive learning is a state-of-the-art self-supervised paradigm based on Noise Contrastive Estimation (NCE) \citep{Gutmann2010} whose most successful applications rely on instance discrimination \citep{He2020, Chen2020b, Yang2020, Han2020b}. Pairs of views from same images or videos are generated by carefully designed data augmentations \citep{Chen2020b, Tian2020a, Feichtenhofer2021}. Elements from the same pairs are called \emph{positives} and their representations are pulled together to learn view invariant features. Other instances called \emph{negatives} are considered as noise and their representations are pushed away from positives. Frameworks based on contrastive learning paradigm require a procedure to sample positives and negatives to learn a good data representation. Videos add the time dimension that offers more possibilities than images to generate positives such as sampling different clips as positives \citep{Feichtenhofer2021, Qian2021}, using different temporal context \citep{Pan2021, Recasens2021, Dave2022}.

A large number of negatives is essential \citep{VanDenOord2018} and various strategies have been proposed to enhance the number of negatives \citep{Chen2020b, Wu2018, He2020, Kalantidis2020}. Sampling hard negatives \citep{Kalantidis2020, Robinson2021, Wu2021, Hu2021, Dwibedi2021} improve the representations but can be harmful if they are semantically false negatives which causes the ``class collision problem" \citep{Cai2020, Wei2021, Chuang}. 

Other approaches that learn from positive views without negatives have been proposed by predicting pseudo-classes of different views \citep{Caron2020, Caron2021, Toering2022}, minimizing the feature distance of positives \citep{Grill2020, Chen2021a, Feichtenhofer2021} or matching the similarity distribution between views and other instances \citep{Zheng2021}. These methods free the mentioned problem of sampling hard negatives.

Based on the weakness of contrastive learning using negatives, we introduce a self-supervised soft contrastive learning approach called Similarity Contrastive Estimation (SCE), that contrasts positive pairs with other instances and leverages the push of negatives using the inter-instance similarities. Our method computes relations defined as a sharpened similarity distribution between augmented views of a batch. Each view from the batch is paired with a differently augmented query. Our objective function will maintain for each query the relations and contrast its positive with other images or videos. A memory buffer is maintained to produce a meaningful distribution. Experiments on several datasets show that our approach outperforms our contrastive and relational baselines MoCov2 \citep{Chen2020a} and ReSSL \citep{Zheng2021} on images. We also demonstrate using relations for video representation learning is better than contrastive learning.

Our contributions can be summarized as follows:
\begin{itemize}[noitemsep,topsep=0pt,leftmargin=*]
\item We propose a self-supervised soft contrastive learning approach called Similarity Contrastive Estimation (SCE) that contrasts pairs of augmented instances with other instances and maintains relations among instances for either image or video representation learning. 
\item We demonstrate that SCE outperforms on several benchmarks its baselines MoCov2 \citep{Chen2020a} and ReSSL \citep{Zheng2021} on images on the same architecture.
\item We show that our proposed SCE is competitive with the state of the art on the ImageNet linear evaluation protocol and generalizes to several image downstream tasks.
\item We show that our proposed SCE reaches state-of-the-art results for video representation learning by pretraining on the Kinetics400 dataset as we beat or match previous top-1 accuracy for finetuning on HMDB51 and UCF101 for ResNet3D-18 and ResNet3D-50. We also demonstrate it generalizes to several video downstream tasks.
\end{itemize}

\section{Related Work}\label{Related Works}

\subsection{Image Self-Supervised Learning}

\textbf{Early Self-Supervised Learning.} In early works, different \emph{pretext tasks} to perform Self-Supervised Learning have been proposed to learn a good data representation. They consist in transforming the input data or part of it to perform supervision such as: instance discrimination \citep{Dosovitskiy2016}, patch localization \citep{Doersch2015}, colorization \citep{Zhang2016}, jigsaw puzzle \citep{Noroozi2016}, counting \citep{Noroozi2017}, angle rotation prediction \citep{Gidaris2018}.

\textbf{Contrastive Learning.} Contrastive learning is a learning paradigm \citep{VanDenOord2018, Wu2018, DevonHjelm2018a, Tian2020, He2020, Chen2020b, Misra2020, Tian2020a, Caron2020, Grill2020, Dwibedi2021, Hu2021, Wang2021a} that outperformed previously mentioned \emph{pretext tasks}. Most successful methods rely on instance discrimination with a \emph{positive} pair of views from the same image contrasted with all other instances called \emph{negatives}. Retrieving lots of negatives is necessary for contrastive learning \citep{VanDenOord2018} and various strategies have been proposed. MoCo(v2) \citep{He2020, Chen2020a} uses a small batch size and keeps a high number of negatives by maintaining a memory buffer of representations via a momentum encoder. Alternatively, SimCLR \citep{Chen2020b, Chen2020d} and MoCov3 \citep{Chen2021b} use a large batch size without a memory buffer, and without a momentum encoder for SimCLR.

\textbf{Sampler for Contrastive Learning.} All negatives are not equal \citep{Cai2020} and hard negatives, negatives that are difficult to distinguish with positives, are the most important to sample to improve contrastive learning. However, they are potentially harmful to the training because of the ``class collision" problem \citep{Cai2020, Wei2021, Chuang}. Several samplers have been proposed to alleviate this problem such as using the nearest neighbor as positive for NNCLR \citep{Dwibedi2021}. Truncated-triplet \citep{Wang2021a} optimizes a triplet loss using the k-th similar element as negative that showed significant improvement. It is also possible to generate views by adversarial learning as AdCo \citep{Hu2021} showed.

\textbf{Contrastive Learning without negatives.} Various siamese frameworks perform contrastive learning without the use of negatives to avoid the class collision problem. BYOL \citep{Grill2020} trains an online encoder to predict the output of a momentum updated target encoder. SwAV \citep{Caron2020} enforces consistency between online cluster assignments from learned prototypes. DINO \citep{Caron2021} proposes a self-distillation paradigm to match distribution on pseudo class from an online encoder to a momentum target encoder. Barlow-Twins \citep{Zbontar2021} aligns the cross-correlation matrix between two paired outputs to the identity matrix that VICReg \citep{Bardes2022} stabilizes by adding an intra-batch decorrelation loss function.

\textbf{Regularized Contrastive Learning.} Several works regularize contrastive learning by optimizing a contrastive objective along with an objective that considers the similarities among instances. CO2 \citep{Wei2021} adds a consistency regularization term that matches the distribution of similarity for a query and its positive. PCL \citep{Li2021} and WCL \citep{Zheng2021b} combines unsupervised clustering with contrastive learning to tighten representations of similar instances.  

\textbf{Relational Learning.} Contrastive learning implicitly learns relations among instances by optimizing alignment and matching a prior distribution \citep{Wang2020, Chen2020c}. ReSSL \citep{Zheng2021} introduces an explicit relational learning objective by maintaining consistency of pairwise similarities between strong and weak augmented views. The pairs of views are not directly aligned which harms the discriminative performance. 

In our work, we optimize a contrastive learning objective using negatives that alleviate class collision by pulling related instances. We do not use a regularization term but directly optimize a soft contrastive learning objective that leverages the contrastive and relational aspects. 

\subsection{Video Self-Supervised Learning}

Video Self-Supervised Learning follows the advances of Image Self-Supervised Learning and often picked ideas from the image modality with adjustment and improvement to make it relevant for videos and make best use of it.

\textbf{Pretext tasks.} As for images, in early works several \emph{pretext tasks} have been proposed on videos. Some were directly picked from images such as rotation \citep{Jing2018}, solving Jigsaw puzzles \citep{Kim2019} but others have been designed specifically for videos. These specific pretext-tasks include predicting motion and appearance \citep{Wang2019}, the shuffling of frame \citep{Lee2017, Misra2016} or clip \citep{Xu2019, Jenni2020} order, predicting the speed of the video \citep{Benaim2020, Yao2020}. These methods have been replaced over time by more performing approaches that are less limited by a specific pretext task to learn a good representation. Recently, TransRank \citep{Duan2022} introduced a new paradigm to perform temporal and spatial pretext tasks prediction on a clip relatively to other transformations to the same clip and showed promising results. 

\textbf{Contrastive Learning.} Video Contrastive Learning \citep{Han2020b, Lorre2020, Yang2020, Pan2021, Qian2021, Qian2021a, Feichtenhofer2021, Recasens2021, Sun2021, Dave2022} has been widely studied in the recent years as it gained interest after its better performance than standard pretext tasks in images. Several works studied how to form positive views from different clips \citep{Han2020b, Qian2021, Feichtenhofer2021, Pan2021} to directly apply contrastive methods from images. CVRL \citep{Qian2021}  extended SimCLR to videos and propose a temporal sampler for creating temporally overlapped but not identical positive views which can avoid spatial redundancy. Also, \citet{Feichtenhofer2021} extended SimCLR, MoCo, SwaV and BYOL to videos and studied the effect of using random sampled clips from a video to form views. They pushed further the study to sample several positives to generalize the Multi-crop procedure introduced for images by \citet{Caron2020}. Some works focused on combining contrastive learning and predicting a pretext task \citep{Piergiovanni2020, Wang2020a, Chen2021, Hu2021, Huang2021, Jenni2021}. To help better represent the time dimension, several approaches were designed to use different temporal context width \citep{Pan2021, Recasens2021, Dave2022} for the different views. 

\textbf{Multi-modal Learning.} To improve self-supervised representation learning, several approaches made use of several modalities to better capture the spatio-temporal information provided by a video. It can be from text \citep{Sun2019,Miech2020}, audio \citep{Alwassel2020, Piergiovanni2020, Recasens2021}, and optical flow \citep{Han2020b, Lorre2020, Han2020a, Piergiovanni2020, Hu2021a, Recasens2021, Toering2022}. 

In our work, we propose a soft contrastive learning objective using only RGB frames that directly generalizes our approach from image with minor changes. To the best of our knowledge, we are the first to introduce the concept of soft contrastive learning using relations for video self-supervised representation learning.

\section{Methodology}\label{Methodology}

\begin{figure*}
\begin{subfigure}[b]{0.5\textwidth}
 \centering
\includegraphics[width=0.6\textwidth]{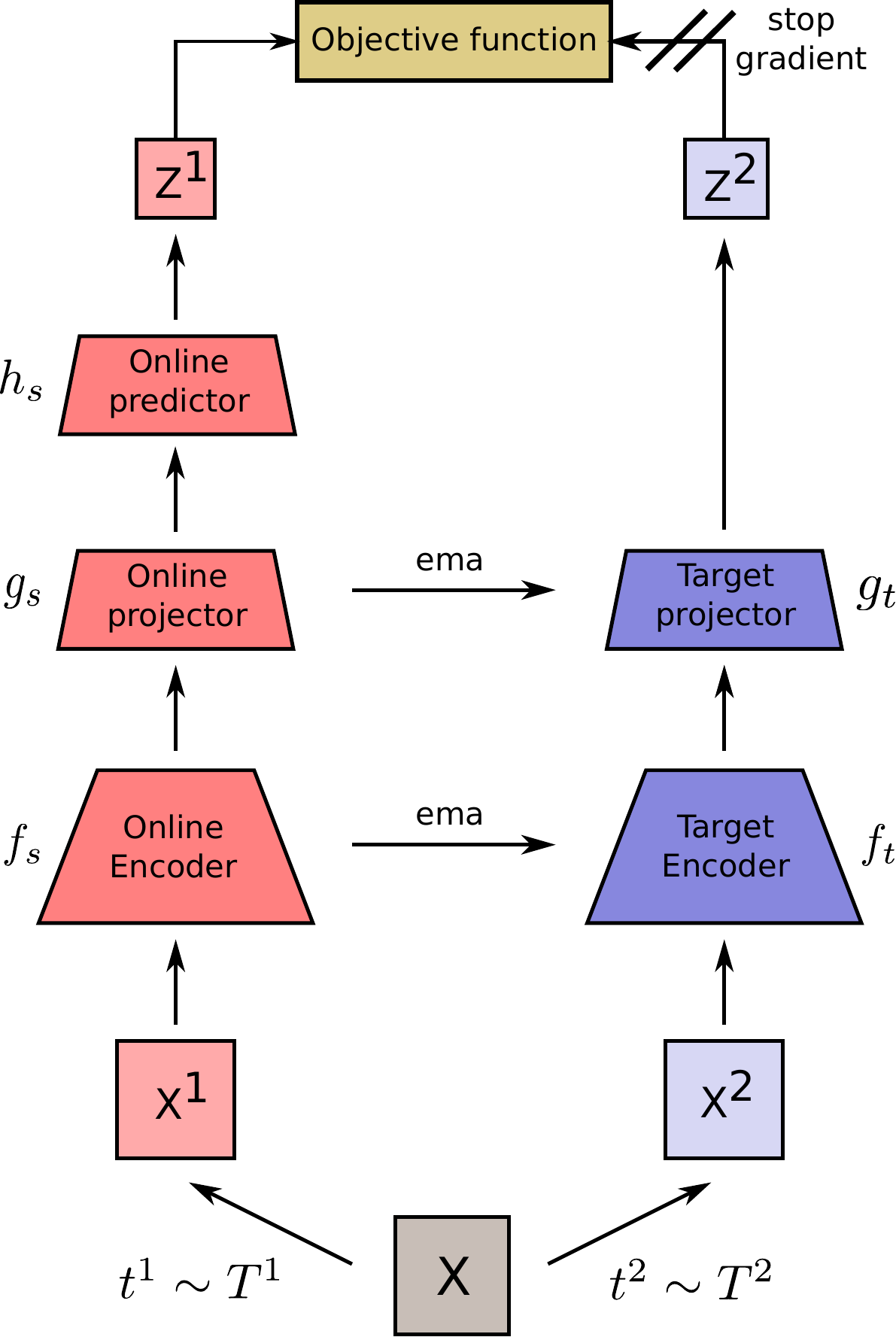}
 \caption{Siamese pipeline}
 \label{fig:siamese}
\end{subfigure}
\begin{subfigure}[b]{0.5\textwidth}
 \centering
\includegraphics[width=0.80\textwidth]{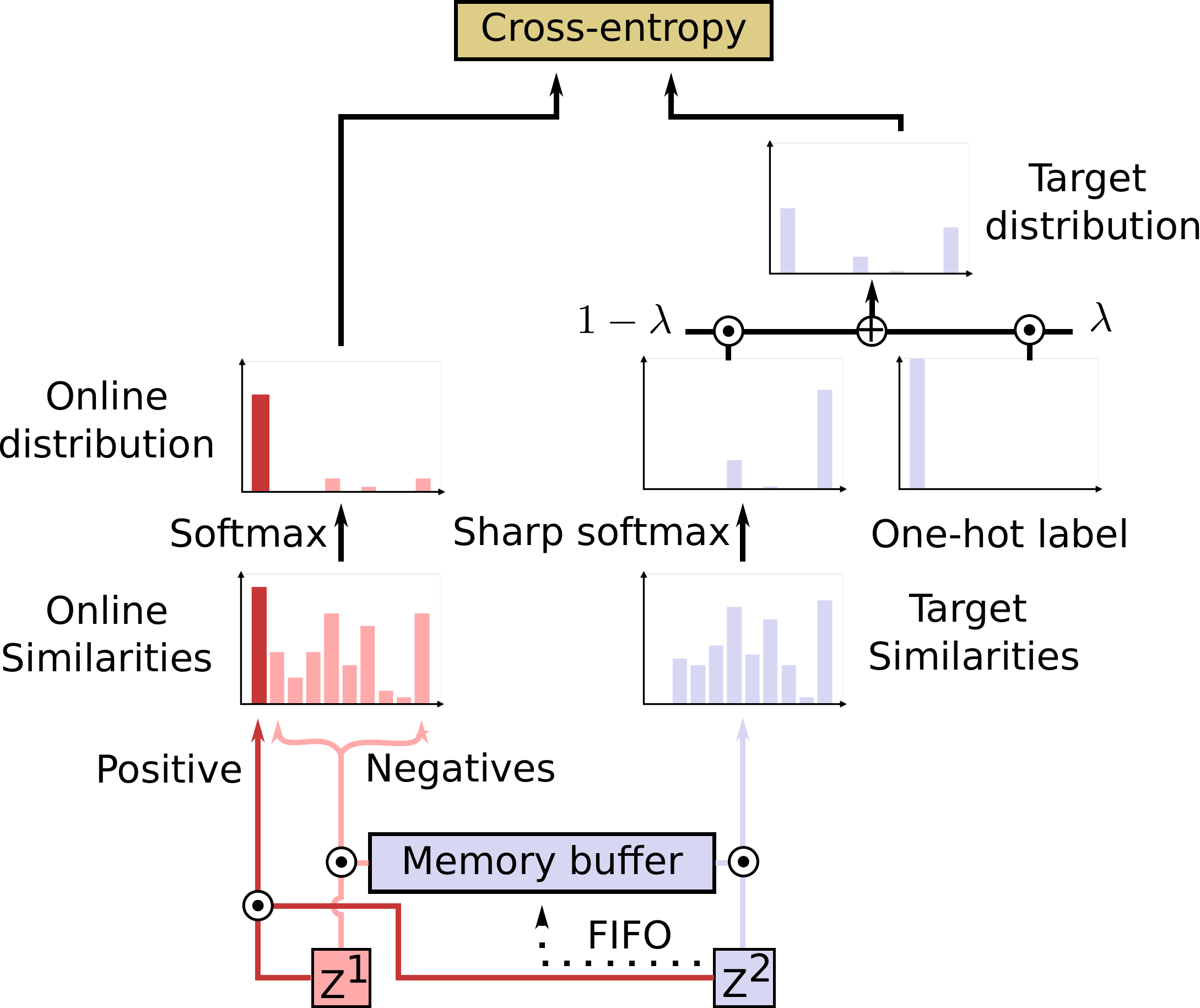}
 \caption{SCE objective function}
 \label{fig:sce}
\end{subfigure}
\caption{SCE follows a siamese pipeline illustrated in \cref{fig:siamese}. A batch $\mathbf{x}$ of images is augmented with two different data augmentation distributions $T^1$ and $T^2$  to form $\mathbf{x^1} = t^1(\mathbf{x})$ and $\mathbf{x^2} = t^2(\mathbf{x})$ with $t^1 \sim T^1$ and $t^2 \sim T^2$. The representation $\mathbf{z^1}$ is computed through an online encoder $f_s$, projector $g_s$ and optionally a predictor $h_s$ such as $\mathbf{z^1} = h_s(g_s(f_s(\mathbf{x^1})))$. A parallel target branch updated by an exponential moving average of the online branch, or ema, computes $\mathbf{z^2} = g_t(f_t(\mathbf{x^2}))$ with $f_t$ and $g_t$ the target encoder and projector. In the objective function of SCE illustrated in \cref{fig:sce}, $\mathbf{z^2}$ is used to compute the inter-instance target distribution by applying a sharp softmax to the cosine similarities between $\mathbf{z^2}$ and a memory buffer of representations from the momentum branch. This distribution is mixed via a $1 - \lambda$ factor with a one-hot label factor $\lambda$ to form the target distribution. Similarities between $\mathbf{z^1}$ and the memory buffer plus its positive in $\mathbf{z^2}$ are also computed. The online distribution is computed via softmax applied to the online similarities. The objective function is the cross entropy between the target and the online distributions.}
\label{fig:pipeline}
\end{figure*}

In this section, we will introduce our baselines: MoCov2 \citep{Chen2020a} for the contrastive aspect and ReSSL \citep{Zheng2021} for the relational aspect. We will then present our self-supervised soft contrastive learning approach called Similarity Contrastive Estimation (SCE). All these methods share the same architecture illustrated in \cref{fig:siamese}. We provide the pseudo-code of our algorithm in Appendix A.

\subsection{Contrastive and Relational Learning}\label{Contrastive and Relational Learning}
Consider $\mathbf{x}=\{\mathbf{x_k}\}_{k\in \{1, ..., N\}}$ a batch of $N$ images. Siamese momentum methods based on Contrastive and Relational learning, such as MoCo \citep{He2020} and ReSSL \citep{Zheng2021} respectively, produce two views of $\mathbf{x}$, $\mathbf{x^1} = t^1(\mathbf{x})$ and $\mathbf{x^2} = t^2(\mathbf{x})$, from two data augmentation distributions $T^1$ and $T^2$ with $t^1 \sim T^1$ and $t^2 \sim T^2$.  For ReSSL, $T^2$ is a weak data augmentation distribution compared to $T^1$ to maintain relations. $\mathbf{x^1}$ passes through an online network $f_s$ followed by a projector $g_s$ to compute $\mathbf{z^1} = g_s(f_s(\mathbf{x^1}))$. A parallel target branch containing a projector $g_t$ and an encoder $f_t$ computes $\mathbf{z^2} = g_t(f_t(\mathbf{x^2}))$. $\mathbf{z^1}$ and $\mathbf{z^2}$ are both $l_2$-normalized. 

The online branch parameters $\theta_s$ are updated by gradient ($\nabla$) descent to minimize a loss function $ \mathcal{L}$. The target branch parameters $\theta_t$ are updated at each iteration by exponential moving average of the online branch parameters with the \emph{momentum value} $m$, also called \emph{keep rate}, to control the update such as:
\begin{align}
    & \theta_s \leftarrow optimizer(\theta_s, \nabla_{\theta_s}\mathcal{L}), \\
    & \theta_t \leftarrow m\theta_t + (1 - m) \theta_s.
\end{align}

MoCo uses the InfoNCE loss, a similarity based function scaled by the temperature $\tau$ that maximizes agreement between the positive pair and push negatives away:

\begin{equation}
  L_{InfoNCE} = - \frac{1}{N} \sum_{i=1}^N \log\left(\frac{\exp(\mathbf{z^1_i} \cdot \mathbf{z^2_i} / \tau)}{\sum_{j=1}^N\exp(\mathbf{z^1_i} \cdot \mathbf{z^2_j} / \tau)}\right).
  \label{eq:1}
\end{equation}

ReSSL computes a target similarity distribution $\mathbf{s^2}$, that represents the relations between weak augmented instances, and the distribution of similarity $\mathbf{s^{1}}$ between the strongly augmented instances with the weak augmented ones. Temperature parameters are applied to each distribution: $\tau$ for $\mathbf{s^{1}}$ and $\tau_m$ for $\mathbf{s^{2}}$ with $\tau > \tau_m$ to eliminate noisy relations. The loss function is the cross-entropy between $\mathbf{s^{2}}$ and $\mathbf{s^{1}}$:

\begin{equation}
  s^{1}_{ik} = \frac{\mathbbm{1}_{i \neq k} \cdot \exp(\mathbf{z^1_i} \cdot \mathbf{z^2_k} / \tau)}{\sum_{j=1}^{N}{\mathbbm{1}}_{i \neq j} \cdot \exp(\mathbf{z^1_i} \cdot \mathbf{z^2_j} / \tau)},
  \label{eq:2}
 \end{equation}
 
 \begin{equation}
  s^{2}_{ik} = \frac{{\mathbbm{1}}_{i \neq k} \cdot \exp(\mathbf{z^2_i} \cdot \mathbf{z^2_k} / \tau_m)}{\sum_{j=1}^{N}{\mathbbm{1}}_{i \neq j} \cdot \exp(\mathbf{z^2_i} \cdot \mathbf{z^2_j} / \tau_m)},
  \label{eq:3}
\end{equation}

\begin{equation}
  L_{ReSSL} = - \frac{1}{N} \sum_{i=1}^N\sum_{\substack{k=1 \\ k\neq i}}^N s^{2}_{ik} \log\left(s^{1}_{ik}\right).
  \label{eq:4}
\end{equation}

A memory buffer of size $M >> N$ filled by $\mathbf{z^2}$ is maintained for both methods.

\subsection{Similarity Contrastive Estimation}\label{Similarity Contrastive Estimation}

Contrastive Learning methods damage relations among instances which Relational Learning correctly build. However Relational Learning lacks the discriminating features that contrastive methods can learn. If we take the example of a dataset composed of cats and dogs, we want our model to be able to understand that two different cats share the same appearance but we also want our model to learn to distinguish details specific to each cat. Based on these requirements, we propose our approach called Similarity Contrastive Estimation (SCE).

We argue that there exists a true distribution of similarity $\mathbf{w_i^*}$ between a query $\mathbf{q_i}$ and the instances in a batch of $N$ images $\mathbf{x}=\{\mathbf{x_k}\}_{k\in \{1, ..., N\}}$, with $\mathbf{x_i}$ a positive view of $\mathbf{q_i}$. If we had access to $\mathbf{w_i^*}$, our training framework would estimate the similarity distribution $\mathbf{p_i}$ between $\mathbf{q_i}$ and all instances in $\mathbf{x}$, and minimize the cross-entropy between $\mathbf{w_i^*}$ and $\mathbf{p_i}$ which is a soft contrastive learning objective:

\begin{equation}
L_{SCE^*} = - \frac{1}{N}\sum_{i=1}^N\sum_{k=1}^N w^*_{ik}\log\left(p_{ik}\right).
\label{eq:5}
\end{equation}

$L_{SCE^*}$ is a soft contrastive approach that generalizes InfoNCE and ReSSL objectives. InfoNCE is a hard contrastive loss that estimates $\mathbf{w_i^*}$ with a one-hot label and ReSSL estimates $\mathbf{w_i^*}$ without the contrastive component. 

We propose an estimation of $\mathbf{w_i^*}$ based on contrastive and relational learning. We consider $\mathbf{x^1} = t^1(\mathbf{x})$ and $\mathbf{x^2} = t^2(\mathbf{x})$ generated from $\mathbf{x}$ using two data augmentations $t^1 \sim T^1$ and $t^2 \sim T^2$. Both augmentation distributions should be different to estimate different relations for each view as shown in \cref{Ablative study}. We compute $\mathbf{z^1} = h_s(g_s(f_s(\mathbf{x^1})))$ from the online encoder $f_s$, projector $g_s$ and optionally a predictor $h_s$ \citep{Grill2020, Chen2021b}). We also compute $\mathbf{z^2} = g_t(f_t(\mathbf{x^2}))$ from the target encoder $f_t$ and projector $g_t$. $\mathbf{z^1}$ and $\mathbf{z^2}$ are both $l_2$-normalized. 

The similarity distribution $\mathbf{s^2_i}$ that defines relations between the query and other instances is computed via the \cref{eq:3}. The temperature $\tau_m$ sharpens the distribution to only keep relevant relations. A weighted positive one-hot label is added to $\mathbf{s^2_i}$ to build the target similarity distribution $\mathbf{w^2_i}$:

\begin{equation}
  w^2_{ik} = \lambda \cdot \mathbbm{1}_{i=k} + (1 - \lambda) \cdot s^2_{ik}.
  \label{eq:7}
\end{equation}

The online similarity distribution $\mathbf{p^1_i}$ between $\mathbf{z^1_i}$ and $\mathbf{z^2}$, including the target positive representation in opposition with ReSSL, is computed and scaled by the temperature $\tau$ with $\tau > \tau_m$ to build a sharper target distribution:

\begin{equation}
  p^1_{ik} = \frac{\exp(\mathbf{z^1_i} \cdot \mathbf{z^2_k} / \tau)}{\sum_{j=1}^{N}\exp(\mathbf{z^1_i} \cdot \mathbf{z^2_j} / \tau)}.
  \label{eq:8}
\end{equation}

The objective function illustrated in \cref{fig:sce} is the cross-entropy between each $\mathbf{w^2}$ and $\mathbf{p^1}$:
\begin{equation}
  L_{SCE} = - \frac{1}{N} \sum_{i=1}^N\sum_{k=1}^N w^2_{ik} \log\left(p^1_{ik}\right).
  \label{eq:9}
\end{equation}

The loss can be symmetrized by passing $\mathbf{x^1}$ and $\mathbf{x^2}$ through the momentum and online encoders and averaging the two losses computed.

A memory buffer of size $M >> N$ filled by $\mathbf{z^2}$ is maintained to better approximate the similarity distributions. 

The following proposition explicitly shows that SCE optimizes a contrastive learning objective while maintaining inter-instance relations:
\begin{prop}
\label{prop:1}
$L_{SCE}$ defined in \cref{eq:9} can be written as:
\begin{equation}
  L_{SCE} = \lambda \cdot L_{InfoNCE} + \mu \cdot L_{ReSSL} +  \eta \cdot L_{Ceil},
  \label{eq:10}
\end{equation}
with $\mu = \eta = 1 - \lambda$ and 
\begin{equation*}
    L_{Ceil} = - \frac{1}{N} \sum_{i=1}^{N}\log\left(\frac{\sum_{j=1}^{N}\mathbbm{1}_{i \neq j} \cdot \exp(\mathbf{z^1_i} \cdot \mathbf{z^2_j} / \tau)}{\sum_{j=1}^{N} \exp(\mathbf{z^1_i} \cdot \mathbf{z^2_j} / \tau)}\right).
\end{equation*}
\end{prop}

The proof separates the positive term and negatives. It can be found in Appendix B. $L_{Ceil}$ leverages how similar the positives should be with hard negatives. Because our approach is a soft contrastive learning objective, we optimize the formulation in \cref{eq:9} and have the constraint $\mu = \eta = 1 - \lambda$. It frees our implementation from having three losses to optimize with two hyperparameters $\mu$ and $\eta$ to tune. Still, we performed a small study of the objective defined in \cref{eq:10} without this constraint to check if $L_{Ceil}$ improves results in \cref{Ablative study}.

\begin{table*}
    \centering
    \small
    \begin{tabular}{lccccc}
        \toprule
        Parameter                           & \emph{weak} & \emph{strong} & \emph{strong-$\alpha$} & \emph{strong-$\beta$} & \emph{strong-$\gamma$} \\ \midrule
        Random crop probability             & 1    & 1      & 1       & 1        & 1 \\ 
        Flip probability                    & 0.5  & 0.5    & 0.5     & 0.5      & 0.5 \\
        Color jittering probability         & 0.   & 0.8    & 0.8     & 0.8      & 0.8 \\
        Brightness adjustment max intensity & -    & 0.4    & 0.4     & 0.4      & 0.4 \\  
        Contrast adjustment max intensity   & -    & 0.4    & 0.4     & 0.4      & 0.4 \\
        Saturation adjustment max intensity & -    & 0.4    & 0.2     & 0.2      & 0.2 \\
        Hue adjustment max intensity        & -    & 0.1    & 0.1     & 0.1      & 0.1 \\
        Color dropping probability          & 0.   & 0.2    & 0.2     & 0.2      & 0.2 \\
        Gaussian blurring probability       & 0.   & 0.5    & 1.      & 0.1      & 0.5 \\
        Solarization probability            & 0.   & 0.     & 0.      & 0.2      & 0.2 \\ \bottomrule    
    \end{tabular}
    \caption{Different distributions of data augmentations applied to SCE. The \emph{weak} distribution is the same as ReSSL \citep{Zheng2021}, \emph{strong} is the standard contrastive data augmentation \citep{Chen2020b}. The \emph{strong-$\alpha$} and \emph{strong-$\beta$} are two distributions introduced by BYOL \citep{Grill2020}. Finally, \emph{strong-$\gamma$} is a mix between \emph{strong-$\alpha$} and \emph{strong-$\beta$}.}
    \label{tab:augmentations}
\end{table*}

\begin{table*}
  \centering
  \small
  \begin{tabular}{cccccccccccc}
    \toprule
    $\lambda$ & 0.    & 0.1   & 0.2   & 0.3   & 0.4   & 0.5   & 0.6   & 0.7   & 0.8   & 0.9   & 1.0 \\
    Top-1     & 81.5 & 81.8 & 82.5 & 82.8 & 82.9 & \textbf{82.9} & 82.2 & 81.6 & 81.8 & 81.8 & 81.1 \\
    \bottomrule
  \end{tabular}
  \caption{Effect of varying $\lambda$ on the Top-1 accuracy on ImageNet100. The optimal $\lambda$ is in [$0.4,0.5$] confirming that learning to discriminate and maintaining relations is best.}
  \label{tab:lambda}
\end{table*}

\section{Empirical study}\label{Empirical study}

In this section, we will empirically prove the relevance of our proposed Similarity Contrastive Estimation (SCE) self-supervised learning approach to learn a good data representation for both images and videos representation learning.

\subsection{Image study}\label{Image study}

In this section, we first make an ablative study of our approach SCE to find the best hyperparameters on images. Secondly, we compare SCE with its baselines MoCov2 \citep{Chen2020a} and ReSSL \citep{Zheng2021} for the same architecture. Finally, we evaluate SCE on the ImageNet Linear evaluation protocol and assess its generalization capacity on various tasks.

\subsubsection{Ablation study}\label{Ablative study}

To make the ablation study, we conducted experiments on ImageNet100 that has a close distribution to ImageNet, studied in \cref{ImageNet Linear Protocol}, with the advantage to require less resources to train. We keep implementation details close to ReSSL \citep{Zheng2021} and MoCov2 \citep{Chen2020a} to ensure fair comparison.

\textbf{Dataset.} ImageNet \citep{Deng2009} is a large dataset with 1k classes, almost 1.3M images in the training set and 50K images in the validation set. ImageNet100 is a selection of 100 classes from ImageNet whose classes have been selected randomly. We took the selected classes from \citep{Tian2020} referenced in Appendix C.

\textbf{Implementation details for pretraining.} We use the ResNet-50 \citep{He2016} encoder and pretrain for 200 epochs. We apply by default \emph{strong} and \emph{weak} data augmentations defined in \cref{tab:augmentations}. We do not use a predictor and we do not symmetry the loss by default. Specific hyper-parameter details can be found in Appendix D.1.

\textbf{Evaluation protocol.} To evaluate our pretrained encoders,  we train a linear classifier following \citep{Chen2020a, Zheng2021} that is detailed in Appendix D.1.

\begin{table}
    \centering
    \footnotesize
    \begin{tabular}{l|ccc|cc}
    \toprule
    Method & \multicolumn{3}{c|}{Loss coefficients} & \multicolumn{2}{c}{Top-1} \\
    & $\lambda $ & $\mu$ & $\eta$ & $\tau_m = 0.05$ & $\tau_m = 0.07$ \\ \midrule
    InfoNCE & 1.        & 0.    & 0.         & 81.11 & 81.11 \\
            & 0.5       & 0.5   & 0.         & 82.80  & 82.49 \\  
    SCE     & 0.5       & 0.5   & 0.5        & \textbf{82.94} & \textbf{83.37} \\
    ReSSL   & 0.        & 1.    & 0.         & 80.79 & 78.35 \\ 
            & 0.        & 1.    & 1.         & 81.53 & 79.64 \\ \bottomrule
\end{tabular}
\caption{Effect of loss coefficients in \cref{eq:10} on the Top-1 accuracy on ImageNet100. $L_{Ceil}$ consistently improves performance that varies given the temperature parameters.}
\label{tab:loss}
\end{table}

\textbf{Leveraging contrastive and relational learning.} 
SCE defined in \cref{eq:7} leverages contrastive and relational learning via the $\lambda$ coefficient. We studied the effect of varying the $\lambda$ coefficient on ImageNet100. Temperature parameters are set to $\tau = 0.1$ and $\tau_m = 0.05$. We report the results in \cref{tab:lambda}. Performance increases with $\lambda$ from $0$ to $0.5$ after which it starts decreasing. The best $\lambda$ is inside $[0.4, 0.5]$ confirming that balancing the contrastive and relational aspects provides better representation. In next experiments, we keep $\lambda = 0.5$.

We performed a small study of the optimization of \cref{eq:10} by removing $L_{ceil}$ ($\eta = 0$) to validate the relevance of our approach for $\tau = 0.1$ and $\tau_m\in\{0.05, 0.07\}$. The results are reported in \cref{tab:loss}. Adding the term $L_{ceil}$ consistently improves performance, empirically proving that our approach is better than simply adding $L_{InfoNCE}$ and $L_{ReSSL}$. This performance boost varies with temperature parameters and our best setting improves by $+0.9$ percentage points (p.p.) in comparison with adding the two losses.

\begin{table}
    \centering
    \begin{tabular}{cccc}
    \toprule 
    Online aug & Teacher aug & Sym & top-1 \\ \midrule
    \emph{strong}     & \emph{weak}        & no  & 82.9 \\
    \emph{strong-$\gamma$} & \emph{weak}   & no  & \textbf{83.0} \\
    \emph{weak}       & \emph{strong}      & no  & 73.4 \\
    \emph{strong}     & \emph{strong}      & no  & 80.5 \\
    \emph{strong-$\alpha$} & \emph{strong-$\beta$} & no  & 80.7 \\ \midrule
    \emph{strong}     & \emph{weak}        & yes & 83.7 \\
    \emph{strong}     & \emph{strong}      & yes & 83.0 \\
    \emph{strong-$\alpha$} & \emph{strong-$\beta$} & yes & \textbf{84.2} \\ 
    \bottomrule
    \end{tabular}
    \caption{Effect of using different distributions of data augmentations for the two views and of the loss symmetrization on the Top-1 accuracy on ImageNet100. Using a \emph{weak} view for the teacher without symmetry is necessary to obtain good relations. With loss symmetry, asymmetric data augmentations improve the results, with the best obtained using strong-$\alpha$ and  \emph{strong-$\beta$}.}
    \label{tab:aug}
\end{table}

\textbf{Asymmetric data augmentations to build the similarity distributions.}
Contrastive learning approaches use strong data augmentations \citep{Chen2020b} to learn view invariant features and prevent the model to collapse. However, these strong data augmentations shift the distribution of similarities among instances that SCE uses to approximate $w_i^*$ in \cref{eq:7}. We need to carefully tune the data augmentations to estimate a relevant target similarity distribution. We listed different distributions of data augmentations in \cref{tab:augmentations}. The \emph{weak} and \emph{strong} augmentations are the same as described by ReSSL \citep{Zheng2021}. \emph{strong-$\alpha$} and \emph{strong-$\beta$} have been proposed by BYOL \citep{Grill2020}. \emph{strong-$\gamma$} combines \emph{strong-$\alpha$} and \emph{strong-$\beta$}.

We performed a study in \cref{tab:aug} on which data augmentations are needed to build a proper target distribution for the non-symmetric and symmetric settings. We report the Top-1 accuracy on Imagenet100 when varying the data augmentations applied on the online and target branches of our pipeline. For the non-symmetric setting, SCE requires the target distribution to be built from a \emph{weak} augmentation distribution that maintains consistency across instances.

\begin{table}
  \centering
    \begin{tabular}{cc|cc}
    \toprule 
    \multicolumn{2}{c|}{$\tau = 0.1$} & \multicolumn{2}{c}{$\tau = 0.2$} \\
    $\tau_m$ & Top-1 & $\tau_m$ & Top-1 \\ \midrule
    0.03 & 82.3 & 0.03 & \textbf{81.3} \\ 
    0.04 & 82.5 & 0.04 & 81.2 \\
    0.05 & 82.9 & 0.05 & 81.2\\
    0.06 & 82.5 & 0.06 & 81.2 \\ 
    0.07 & \textbf{83.4} & 0.07 & 81.1 \\ 
    0.08 & 82.7 & 0.08 & 80.9 \\ 
    0.09 & 82.5 & 0.09 & 81.2\\ 
    0.10 & 82.1 & 0.10 & 81.2 \\ \bottomrule
    \end{tabular}
    \caption{Effect of varying the temperature parameters $\tau_m$ and $\tau$ on the Top-1 accuracy on ImageNet100. $\tau_m$ is lower than $\tau$ to produce a sharper target distribution without noisy relations. SCE does not collapse when $\tau_m \rightarrow \tau$.}
    \label{tab:temperature}
\end{table}

\begin{table*}
\centering
\small
    \begin{tabular}{lccccccc}
    \toprule 
    Method & ImageNet & ImageNet100 & Cifar10 & Cifar100 & STL10 & Tiny-IN \\ \midrule
    MoCov2 \citep{Chen2020a}  & 67.5 & -  & - & - & - & - \\
    MoCov2 [*]  & 68.8 & 80.5  & 87.6 & 61.0 & 86.5 & 45.9 \\
    ReSSL \citep{Zheng2021} & 69.9 & - & 90.2  & 63.8 & 88.3 & 46.6 \\ 
    ReSSL [*] & 70.2 & 81.6 & 90.2  & 64.0 & 89.1 & 49.5 \\ 
    \textbf{SCE (Ours)}  & \textbf{70.5} & \textbf{83.4} & \textbf{90.3} & \textbf{65.5} & \textbf{89.9} & \textbf{51.9} \\ \bottomrule
    \end{tabular}
    \caption{Comparison of SCE with its baselines MoCov2 \citep{Chen2020a} and ReSSL \citep{Zheng2021} on the Top-1 Accuracy on various datasets. SCE outperforms on all benchmarks its baselines. [*] denotes our reproduction.}
    \label{tab:baselines}
\end{table*}

Once the loss is symmetrized, asymmetry with strong data augmentations has better performance. Indeed, using \emph{strong-$\alpha$} and \emph{strong-$\beta$} augmentations is better than using \emph{weak} and \emph{strong} augmentations, and same \emph{strong} augmentations has lower performance. We argue symmetrized SCE requires asymmetric data augmentations to produce different relations for each view to make the model learn more information. The effect of using stronger augmentations is balanced by averaging the results on both views. Symmetrizing the loss boosts the performance as for \citep{Grill2020, Chen2021a}.

\textbf{Sharpening the similarity distributions.} 
The temperature parameters sharpen the distributions of similarity exponentially. SCE uses the temperatures $\tau_m$ and $\tau$ for the target and online similarity distributions with $\tau_m < \tau$ to guide the online encoder with a sharper target distribution. We made a temperature search on ImageNet100 by varying $\tau$ in $\{0.1, 0.2\}$ and $\tau_m$ in $\{0.03, ..., 0.10\}$. The results are in \cref{tab:temperature}. We found the best values $\tau_m = 0.07$ and $\tau = 0.1$ proving SCE needs a sharper target distribution. In Appendix E, this parameter search is done for other datasets used in comparison with our baselines. Unlike ReSSL \citep{Zheng2021}, SCE does not collapse when $\tau_m \rightarrow \tau$ thanks to the contrastive aspect. Hence, it is less sensitive to the temperature choice.

\subsubsection{Comparison with our baselines}\label{Baselines}

We compared on 6 datasets how SCE performs against its baselines. We keep similar implementation details to ReSSL \citep{Zheng2021} and MoCov2 \citep{Chen2020a} for fair comparison.

\textbf{Small datasets.} Cifar10 and Cifar100 \citep{Krizhevsky2009} have 50K training images, 10K test images, $32 \times 32$ resolution and 10-100 classes respectively. \textbf{Medium datasets.} STL10 \citep{Coates2011} has a $96 \times 96$ resolution, 10 classes, 100K unlabeled data, 5k labeled training images and 8K test images. Tiny-Imagenet \citep{Abai2020} is a subset of ImageNet with $64 \times 64$ resolution, 200 classes, 100k training images and 10K validation images.

\textbf{Implementation details.} Architecture implementation details can be found in Appendix D.1. For MoCov2, we use $\tau = 0.2$ and for ReSSL their best $\tau$ and $\tau_m$ reported \citep{Zheng2021}. For SCE, we use the best temperature parameters from \cref{Ablative study} for ImageNet and ImageNet100 and from Appendix E for the other datasets. The same architecture for all methods is used except for MoCov2 on ImageNet that kept the ImageNet100 projector to improve results.

Results are reported in \cref{tab:baselines}. Our baselines reproduction is validated as results are better than those reported by the authors. SCE outperforms its baselines on all datasets proving that our method is more efficient to learn discriminating features on the pretrained dataset. We observe that our approach outperforms more significantly ReSSL on smaller datasets than ImageNet, suggesting that it is more important to learn to discriminate among instances for these datasets. SCE has promising applications to domains with few data such as in medical applications. 

\begin{table*}
    \centering
    \small
    \begin{tabular}{lcccc}
    \toprule
        Method     & 100  & 200  & 300  & 800-1000 \\ \midrule
        SimCLR \citep{Chen2020b}    & 66.5 & 68.3 & -     & 70.4 \\
        MoCov2 \citep{Chen2021a}     & 67.4 & 69.9 & -     & 72.2 \\
        SwaV \citep{Caron2020}      & 66.5 & 69.1 & -     & 71.8 \\
        BYOL \citep{Grill2020}      & 66.5 & 70.6 & 72.5  & 74.3 \\
        Barlow-Twins\citep{Zbontar2021} & - & - & 71.4 & 73.2 \\
        AdCo \citep{Hu2021}         & -    & 68.6  & - & 72.8 \\
        ReSSL \citep{Zheng2021}     & -    & 71.4 & -     & -    \\
        WCL \citep{Zheng2021b}       & 68.1 & 70.3 & -     & 72.2 \\
        VICReg \citep{Bardes2022}  & - & - & - & 73.2 \\
        UniGrad \citep{Tao2022}   & \underline{70.3} & -    & -     & -    \\
        MoCov3 \citep{Chen2021b}    & 68.9 & -    & \underline{72.8}  & 74.6 \\
        NNCLR \citep{Dwibedi2021}     & 69.4 & 70.7 & -     & 75.4 \\
        Triplet \citep{Wang2021a} & - & \textbf{73.8} & - & \textbf{75.9} \\ \midrule
        \textbf{SCE (Ours)} & \textbf{72.1} & \underline{72.7} & \textbf{73.3}  & 74.1 \\ \bottomrule
    \end{tabular}
    \caption{State-of-the-art results on the Top-1 Accuracy on ImageNet under the linear evaluation protocol at different pretraining epochs: 100, 200, 300, 800+. SCE is Top-1 at 100 epochs and Top-2 for 200 and 300 epochs. For 800+ epochs, SCE has lower performance than several state-of the-art methods. Results style: \textbf{best}, \underline{second best}.}
    \label{tab:sota}
\end{table*}

\begin{table}
\centering
\small
    \begin{tabular}{lccc}
    \toprule 
    Method      & Epochs & Top-1 \\ \midrule
    UniGrad \citep{Tao2022} & 100 & 72.3 \\
    SwaV \citep{Caron2020}       & 200    & 72.7 \\
    AdCo \citep{Hu2021} & 200 & 73.2 \\
    WCL \citep{Zheng2021b}        & 200    & 73.3 \\
    Triplet \citep{Wang2021a} & 200    & 74.1 \\
    ReSSL \citep{Zheng2021}      & 200    & 74.7 \\
    WCL \citep{Zheng2021b}         & 800    & 74.7 \\
    SwaV \citep{Caron2020}       & 800    & 75.3 \\
    DINO \citep{Caron2021}        & 800    & 75.3 \\
    UniGrad \citep{Tao2022} & 800 & 75.5 \\
    NNCLR \citep{Dwibedi2021}      & 1000   & 75.6 \\
    AdCo \citep{Hu2021} & 800 & 75.7 \\ \midrule
    \textbf{SCE (ours)}  & 200    & 75.4 \\ \bottomrule
    \end{tabular}
\caption{State-of-the-art results on the Top-1 Accuracy on ImageNet under the linear evaluation protocol with multi-crop. SCE is competitive with the best state-of-the-art methods by pretraining for only 200 epochs instead of 800+.}
\label{tab:sota-multi-crop}
\end{table}

\subsubsection{ImageNet Linear Evaluation}\label{ImageNet Linear Protocol}

We compare SCE on the widely used ImageNet linear evaluation protocol with the state of the art. We scaled our method using a larger batch size and a predictor to match state-of-the-art results \citep{Grill2020, Chen2021b}.

\textbf{Implementation details.} We use the ResNet-50 \citep{He2016} encoder, apply \emph{strong-$\alpha$} and \emph{strong-$\beta$} augmentations defined in \cref{tab:augmentations}. We follow the same training hyperparameters used by \citep{Chen2021b} and detailed in Appendix D.2. The loss is symmetrized and we keep the best hyperparameters from \cref{Ablative study}: $\lambda = 0.5$, $\tau = 0.1$ and $\tau_m = 0.07$.

\textbf{Multi-crop setting.} We follow \citep{Hu2021} setting and sample 6 different views detailed in Appendix D.2.

\textbf{Evaluation protocol.} We follow the protocol defined by \citep{Chen2021b} and detailed in Appendix D.2.

We evaluated SCE at epochs 100, 200, 300 and 1000 on the Top-1 accuracy on ImageNet to study the efficiency of our approach and compare it with the state of the art in \cref{tab:sota}. At 100 epochs, SCE reaches $\mathbf{72.1\%}$ up to $\mathbf{74.1\%}$ at 1000 epochs. Hence, SCE has a fast convergence and few epochs of training already provides a good representation. SCE is the Top-1 method at 100 epochs and Top-2 for 200 and 300 epochs proving the good quality of its representation for few epochs of pretraining. 

At 1000 epochs, SCE is below several state-of-the art results. We argue that SCE suffers from maintaining a $\lambda$ coefficient to $0.5$ and that relational or contrastive aspects do not have the same impact at the beginning and at the end of pretraining. A potential improvement would be using a scheduler on $\lambda$ that varies over time.

We added multi-crop to SCE for 200 epochs of pretraining. It enhances the results but it is costly in terms of time and memory. It improves the results from $72.7$\% to our best result $\mathbf{75.4\%}$ ($\mathbf{+2.7}$\textbf{p.p.}). Therefore, SCE learns from having local views and they should maintain relations to learn better representations. We compared SCE with state-of-the-art methods using multi-crop in \cref{tab:sota-multi-crop}. SCE is competitive with top state-of-the-art methods that trained for 800+ epochs by having slightly lower accuracy than the best method using multi-crop ($-0.3$p.p) and without multi-crop ($-0.5$p.p). SCE is more efficient than other methods, as it reaches state-of-the-art results for fewer pretraining epochs.

\begin{table*}
    \footnotesize
    \centering
    \begin{tabular}{lccccccccccc|c} 
    \toprule
    Method   & Food	& CIFAR10 &	CIFAR100 &	SUN & Cars & Air. &	VOC &	DTD &	Pets &	Caltech &	Flow. & Avg.           \\ \midrule
    SimCLR    & 72.8 &	90.5 & 74.4 & 60.6 & 49.3 &	49.8 & 81.4 & 75.7 & 84.6 &	89.3 & 92.6 & 74.6 \\
    BYOL      & 75.3 & 91.3 & 78.4 & 62.2 & \textbf{67.8} & 60.6 & 82.5 & 75.5 & 90.4 &	94.2 & \textbf{96.1} & 79.5 \\
    NNCLR     & 76.7 & 93.7 & 79.0 & 62.5 & 67.1 &	\textbf{64.1} & 83.0 & 75.5 & \textbf{91.8} &	91.3 & 95.1 & 80 \\
    \textbf{SCE (Ours)} & \textbf{77.7} & \textbf{94.8} & \textbf{80.4} & \textbf{65.3} & 65.7 &	59.6 & \textbf{84.0} & \textbf{77.1} & 90.9 & 92.7 & \textbf{96.1} & \textbf{80.4} \\ \midrule
    Supervised & 72.3 & 93.6 & 78.3 & 61.9 & 66.7 &	61.0 & 82.8 & 74.9 & 91.5 &	\textbf{94.5} & 94.7 & 79.3 \\ \bottomrule
    \end{tabular}
    \caption{Linear classifier trained on popular many-shot recognition datasets in comparison with SimCLR \citep{Chen2020b}, BYOL \citep{Grill2020}, NNCLR \citep{Dwibedi2021} and supervised training. SCE is Top-1 on 7 datasets and in average.}
    \label{tab:many-shot}
\end{table*}

\begin{table*}
\centering
\footnotesize
    \begin{tabular}{lcccc} \toprule
    Method & K = 16 & K = 32 & K = 64 & full \\\midrule
    MoCov2 \citep{Chen2020a} & 76.1 & 79.2 & 81.5 & 84.6 \\
    PCLv2 \citep{Li2021} & 78.3 & 80.7 & 82.7 & 85.4 \\
    ReSSL \citep{Zheng2021} & 79.2 & 82.0 & 83.8 & 86.3 \\
    SwAV \citep{Caron2020} & 78.4 & 81.9 & 84.4 & 87.5 \\
    WCL \citep{Zheng2021b} & \textbf{80.2} & 83.0 & 85.0 & 87.8 \\
    \textbf{SCE (Ours)} & 79.5 & \textbf{83.1} & \textbf{85.5} & \textbf{88.2} \\ \bottomrule
    \end{tabular}
    \caption{Transfer learning on low-shot image classification on Pascal VOC2007. All methods have been pretrained for 200 epochs. SCE is Top-1 when using 32-64-all images per class and Top-2 for 16 images.}
    \label{tab:low-shot-voc}
\end{table*}

\begin{table}
\centering
\footnotesize
    \begin{tabular}{lcc}
    \toprule
    Method            & $AP^{\emph{Box}}$   & $AP^{\emph{Mask}}$ \\ \midrule
    Random            & 35.6             & 31.4 \\
    Rel-Loc \citep{Doersch2015}     & 40.0             & 35.0 \\
    Rot-Pred \citep{Gidaris2018}    & 40.0             & 34.9 \\
    NPID \citep{Wu2018}              & 39.4             & 34.5 \\
    MoCo \citep{He2020}              & 40.9             & 35.5 \\
    MoCov2 \citep{Chen2020a}           & 40.9             & 35.5 \\
    SimCLR \citep{Chen2020b}           & 39.6             & 34.6 \\
    BYOL \citep{Grill2020}             & 40.3             & 35.1 \\
    \textbf{SCE (Ours)} & \underline{41.6} & \underline{36.0} \\
    Triplet \citep{Wang2021a} & \textbf{41.7}    & \textbf{36.2} \\ \midrule
    Supervised        & 40.0             & 34.7 \\ \bottomrule
    \end{tabular}
    \caption{Object detection and Instance Segmentation on COCO \citep{Lin2014} training a Mask R-CNN \citep{He2017}. SCE is Top-2 on both tasks, slightly below Truncated-Triplet \citep{Wang2021a} and better than supervised training. Results style: \textbf{best}, \underline{second best}.}
    \label{tab:object-detection}
\end{table}

\subsubsection{Transfer Learning}
We study the generalization of our proposed SCE on several tasks using our multi-crop checkpoint pretrained for 200 epochs on ImageNet.

\textbf{Low-shot evaluation.} Low-shot transferability of our backbone is evaluated on Pascal VOC2007. We followed the protocol proposed by \citet{Zheng2021}. We select 16, 32, 64 or all images per class to train the classifier. Our results are compared with other state-of-the-art methods pretrained for 200 epochs in \cref{tab:low-shot-voc}. SCE is Top-1 for 32, 64 and all images per class and Top-2 for 16 images per class, proving the generalization of our approach to few-shot learning.

\textbf{Linear classifier for many-shot recognition datasets.} We follow the same protocol as \citet{Grill2020, Ericsson2021} to study many-shot recognition in transfer learning on the datasets FGVC Aircraft \citep{Maji2013}, Caltech-101 \citep{Fei2004}, Standford Cars \citep{Krause2013}, CIFAR-10 \citep{Krizhevsky2009}, CIFAR-100 \citep{Krizhevsky2009}, DTD \citep{Cimpoi2014}, Oxford 102 Flowers \citep{Nilsback2008}, Food-101 \citep{Bossard2014}, Oxford-IIIT Pets \citep{Parkhi2012}, SUN397 \citep{Xiao2010} and Pascal VOC2007 \citep{Everingham2010}. These datasets cover a large variety of number of training images (2k-75k) and number of classes (10-397). We report the Top-1 classification accuracy except for Aircraft, Caltech-101, Pets and Flowers for which we report the mean per-class accuracy and the 11-point MAP for VOC2007.

We report the performance of SCE in comparison with state-of-the-art methods in \cref{tab:many-shot}. SCE outperforms on 7 datasets all approaches. In average, SCE is above all state-of-the-art methods as well as the supervised baseline, meaning SCE is able to generalize to a wide range of datasets.

\textbf{Object detection and instance segmentation.} 
We performed object detection and instance segmentation on the COCO dataset \citep{Lin2014}. We used the pretrained network to initialize a Mask R-CNN \citep{He2017} up to the C4 layer. We follow the protocol of \citet{Wang2021a} and report the Average Precision for detection $AP^{Box}$ and instance segmentation $AP^{Mask}$.

We report our results in \cref{tab:object-detection} and observe that SCE is the second best method after Truncated-Triplet \citep{Wang2021a} on both metrics, by being slightly below their reported results and above the supervised setting. Therefore our proposed SCE is able to generalize to object detection and instance segmentation task beyond what the supervised pretraining can ($\mathbf{+1.6}$\textbf{p.p.} of $AP^{Box}$ and $\mathbf{+1.3}$\textbf{p.p.} of $AP^{Mask})$. 

\subsection{Video study}\label{Video study}

In this section, we first make an ablation study of our approach SCE to find the best hyperparameters on videos. Then, we compare SCE to the state of the art after pretraining on Kinetics400 and assess generalization on various tasks.

\subsubsection{Ablation study}\label{Video Ablative study}

\textbf{Pretraining Dataset.} To make the ablation study, we perform pretraining experiments on Mini-Kinetics200 \citep{Xie2018}, later called Kinetics200 for simplicity. It is a subset of Kinetics400 \citep{Kay2017} meaning they have a close distribution with less resources required on Kinetics200 to train. Kinetics400 is composed of 216k videos for training and 18k for validation for 400 action classes. However, it has been created from Youtube and some videos have been deleted. We use the dataset hosted\footnote{Link to the Kinetics400 dataset hosted by the CVD foundation: \url{https://github.com/cvdfoundation/kinetics-dataset}.} by the CVD foundation.

\textbf{Evaluation Datasets.} To study the quality of our pretrained representation, we perform linear evaluation classification on the Kinetics200 dataset. Also, we finetune on the first split of the UCF101 \citep{Soomro2012} and HMDB51 \citep{Kuehne2011} datasets. UCF101 is an action classification dataset that contains 13k3 different videos for 101 classes and has 3 different training and validation splits. HMDB51 is also an action classification dataset that contains 6k7 different videos from 51 classes with 3 different splits.

\textbf{Pretraining implementation details.} We used the ResNet3D-18 network \citep{Hara2018} following the Slow path of \citet{Feichtenhofer2019}. We kept hyperparameters close to the ones used for ImageNet in \cref{ImageNet Linear Protocol}. More details can be found in Appendix D.3. We pretrain for 200 epochs with a batch size of $512$. The loss is symmetrized. To form two different views from a video, we follow \citet{Feichtenhofer2021} and randomly sample two clips from the video that lasts $2.56$ seconds and keep only 8 frames.

\textbf{Linear evaluation and finetuning evaluation protocols.} We follow \citet{Feichtenhofer2021} and details can be found in Appendix D.3. For finetuning on UCF101 and HMDB51 we only use the first split in ablation study.
    
\begin{table}
    \centering
    \begin{tabular}{lccc}\toprule
        Method & K200 & UCF101 & HMDB51 \\ \midrule
        SCE Baseline & $63.9$ & $86.3$ & $57.0$ \\
        Supervised & $\mathbf{72.0}$ & $\mathbf{87.5}$ & $\mathbf{60.1}$ \\\bottomrule
    \end{tabular}
    \caption{Comparison of our baseline and supervised training on the Kinetics200, UCF101 and HMDB51 Top-1 accuracy. Supervised training is consistently better.}
    \label{tab:video-baseline}
\end{table}

\textbf{Baseline and supervised learning}. We define an SCE baseline which uses the hyperparameters $\lambda=0.5$, $\tau=0.1$, $\tau_m=0.07$. We provide performance of our SCE baseline as well as supervised training in \cref{tab:video-baseline}. We observe that our baseline has lower results than supervised learning with $-8.1$p.p for Kinetics200, $-1.2$p.p for UCF101 and $-3.1$p.p for HMDB51 which shows that our representation has a large margin for improvement. 

\begin{table}
    \centering
    \begin{tabular}{cccc}\toprule
        $\lambda$ & K200 & UCF101 & HMDB51 \\ \midrule
        $0.000$  & $64.2$ & $86.2$ & $57.5$  \\
        $0.125$  & $\mathbf{64.8}$ & $\mathbf{86.9}$ & $\mathbf{58.2}$  \\
        $0.250$  & $64.3$ & $86.7$ & $\mathbf{58.2}$ \\
        $0.375$  & $64.7$ & $86.3$ & $56.8$ \\
        $0.500$  & $63.9$ & $86.3$ & $57.0$ \\
        $0.625$  & $63.4$ & $86.2$ & $55.7$ \\
        $0.750$  & $63.1$ & $85.8$ & $56.2$ \\
        $0.875$  & $62.1$ & $85.7$ & $55.3$ \\
        $1.000$  & $61.9$ & $85.0$ & $55.4$ \\
    \bottomrule
    \end{tabular}
    \label{tab:video-lambda-vary}
    \caption{Effect of varying $\lambda$ on the Kinetics200, UCF101 and HMDB51 Top-1 accuracy. The best $\lambda$ is $0.125$ meaning contrastive and relational leverage increases performance.}
\end{table}

\textbf{Leveraging contrastive and relational learning.} As for the image study, we varied $\lambda$ from the equation \cref{eq:7} in the set $\{0, 0.125, ..., 0.875, 1\}$ to observe the effect of leveraging the relational and contrastive aspects and report results in \cref{tab:video-lambda-vary}. Using relations during pretraining improves the results rather than only optimizing a contrastive learning objective. The performance on Kinetics200, UCF101 and HMDB51 consistently increases by decreasing $\lambda$ from $1$ to $0.25$. The best $\lambda$ obtained is $0.125$. Moreover $\lambda=0$ performs better than $\lambda=1$. These results suggest that for video pretraining with standard image contrastive learning augmentations, relational learning performs better than contrastive learning and leveraging both further improve the quality of the representation.

\begin{table}
    \centering
    \begin{tabular}{cccc}\toprule
        $\tau_m$ & K200 & UCF101 & HMDB51 \\ \midrule
        $0.03$ & $63.4$ & $86.1$ & $56.9$ \\
        $0.04$ & $63.8$ & $\mathbf{86.6}$ & $56.6$ \\ 
        $0.05$ & $\mathbf{64.3}$ & $86.4$ & $\mathbf{57.1}$ \\
        $0.06$ & $64.1$ & $86.2$ & $56.4$ \\
        $0.07$ & $63.9$ & $86.3$ & $57.0$ \\
        $0.08$ & $63.8$ & $85.9$ & $55.8$ \\
    \bottomrule
    \end{tabular}
    \label{tab:video-temp}
    \caption{Effect of varying $\tau_m$ on the Top-1 accuracy on Kinetics200, UCF101 and HMDB51 while maintaining $\tau=0.1$. The best $\tau_m$ is $0.05$ meaning that a sharper target distribution is required.}
\end{table}
        
\textbf{Target temperature variation.} We studied the effect of varying the target temperature with values in the set $\tau_m \in \{0.03, 0.04, ..., 0.08\}$ while maintaining the online temperature $\tau = 0.1$. We report results in \cref{tab:video-temp}. We observe that the best temperature is $\tau_m=0.05$ indicating that a sharper target distribution is required for video pretraining. We also observe that varying $\tau_m$ has a lower impact on performance than varying $\lambda$. 

\begin{table}
    \centering
    \begin{tabular}[]{cccc}\toprule
        strength & K200 & UCF101 & HMDB51 \\ \midrule
        0.50 & $63.9$ & $86.3$ & $57.0$ \\
        0.75 & $64.6$ & $86.8$ & $57.8$ \\
        1.00 & $\mathbf{64.8}$ & $\mathbf{87.0}$ & $\mathbf{58.1}$ \\ \bottomrule
    \end{tabular}
    \caption{Effect of strength for color jittering for \emph{strong-$\alpha$} and \emph{strong-$\beta$} augmentations on the Kinetics200, UCF101 and HMDB51 Top-1 accuracy. Strong color jittering improves performance.}
    \label{tab:video-color}
\end{table}

\begin{table}
    \centering
    \begin{tabular}[]{cccccc}\toprule
        jitter & reverse & diff & K200 & UCF101 & HMDB51 \\ \midrule
        0.0 & 0.0 & 0.0 & $63.9$ & $86.3$ & $57.0$ \\
        0.2 & 0.0 & 0.0 & $64.2$ & $86.4$ & $56.9$ \\
        0.0 & 0.2 & 0.0 & $64.0$ & $85.7$ & $55.4$ \\
        0.0 & 0.0 & 0.2 & $65.4$ & $\mathbf{88.3}$ & $\mathbf{61.4}$ \\
        0.0 & 0.0 & 0.5 & $64.1$ & $87.7$ & $60.8$ \\ \midrule
        \multicolumn{3}{l}{Supervised} & $\mathbf{72.0}$ & $87.5$ & $60.1$ \\\bottomrule
    \end{tabular}
    \caption{Effect of using the temporal augmentations by applying clip duration jittering \emph{jitter}, randomly reversing the order of frames \emph{reverse} or randomly using RGB difference \emph{diff} on the Kinetics200, UCF101 and HMDB51 Top-1 accuracy. The \emph{diff} augmentation consistently improves results on the three benchmarks and outperforms supervised pretraining. The other augmentations unchange or decrease performance in average.}
    \label{tab:video-time}
\end{table}

\textbf{Spatial and temporal augmentations.} We tested varying and adding some data augmentations that generates the pairs of views. As we are dealing with videos, these augmentations can be either spatial or temporal. We define the \emph{jitter} augmentation that jitters by a factor the duration of a clip, \emph{reverse} that randomly reverses the order of frames and \emph{diff} that randomly applies RGB difference on the frames. RGB difference consists in converting the frames to grayscale and subtracting them over time to approximate the magnitude of optical flow. In this work, we consider RGB difference as a data augmentation that is randomly applied during pretraining. In the literature it is often used as a modality to provide better representation quality than RGB frames \citep{Jing2018, Lorre2020, Duan2022}. Here, we only apply it during pretraining as a random augmentation. Evaluation only sees RGB frames.

We tested to increase the color jittering strength in \cref{tab:video-color}. Using a strength of $1.0$ improved our performance on all the benchmarks suggesting that video pretraining requires harder spatial augmentations than images.

\begin{table*}
    \centering
    \small
    \begin{tabular}[]{ccccccc}\toprule
        $\lambda$ & $\tau_m$ & diff & strength & K200 & UCF101 & HMDB51 \\ \midrule
        0.125 & 0.05 & 0.2 & 1.0 & $65.0$ & $87.4$ & $\underline{61.1}$ \\ 
        0.125 & 0.07 & 0.2 & 1.0 & $64.7$ & $88.2$ & $60.6$ \\
        0.500 & 0.05 & 0.2 & 1.0 & $\underline{66.0}$ & $\underline{88.4}$ & $\mathbf{62.0}$ \\
        0.500 & 0.07 & 0.2 & 1.0 & $65.4$ & $\mathbf{88.6}$ &$61.0$ \\ \midrule
        \multicolumn{4}{l}{SCE Baseline} & $63.9$ & $86.3$ & $57.0$ \\
        \multicolumn{4}{l}{Supervised} & $\mathbf{72.0}$ & $87.5$ & $60.1$ \\\bottomrule
    \end{tabular}
    \caption{Effect of combining best hyper-parameters found in the ablation study which are $\lambda=0.125$, $\tau_m=0.05$, \emph{color strength}$=1.0$ and adding randomly time difference  on the Kinetics200, UCF101 and HMDB51 Top-1 accuracy. Using time difference and stronger color jittering increases the optimal $\lambda$ value which indicates contrastive learning is efficient to deal with harder views and helps relational learning. The best value $\tau_m=0.05$ performs favorably for Kinetics200 and HMDB51. Results style: \textbf{best}, \underline{second best}.}
    \label{tab:video-best}
\end{table*}

We tested our defined temporal augmentations with \emph{jitter} of factor $0.2$, meaning sampling clips between $0.80\times 2.56$ and $1.20 \times 2.56$ seconds, randomly applying $reverse$ with $0.2$ probability and randomly applying \emph{diff} with $0.2$ or $0.5$ probability. We report results in \cref{tab:video-time}. Varying the clip duration had no noticeable impact on our benchmarks, but reversing the order of frames decreased the performance on UCF101 and HMDB51. This can be explained by the fact that this augmentation can prevent the model to correctly represent the arrow of time. Finally, applying \emph{diff} with $0.2$ probability considerably improved our performance over our baseline with $\mathbf{+1.5}$\textbf{p.p.} on Kinetics200, $\mathbf{+2.0}$\textbf{p.p.} on UCF101 and $\mathbf{+4.4}$\textbf{p.p.} on HMDB51. It outperforms supervised learning for generalization with $\mathbf{+0.8}$\textbf{p.p.} on UCF101 and $\mathbf{+1.3}$\textbf{p.p.} on HMDB51. Applying more often \emph{diff} decreases performance. These results show that SCE benefits from using views that are more biased towards motion than appearance. We believe that it is particularly efficient to model relations based on motion.

\textbf{Bringing all together.} We studied varying one hyperparameter from our baseline and how it affects performance. In this final study, we combined our baseline with the different best hyperparameters found which are $\lambda=0.125$, $\tau_m=0.05$, color strength $=1.0$ and applying \emph{diff} with $0.2$ probability. We report results in \cref{tab:video-best} and found out that using harder augmentations increased the optimal $\lambda$ value as using $\lambda=0.5$ performs better than $\lambda=0.125$. This indicates that relational learning by itself cannot learn a better representation through positive views that share less mutual information. The contrastive aspect of our approach is proven efficient for such harder positives. We take as best configuration $\lambda=0.5$, $\tau_m=0.05$, \emph{diff} applied with probability $0.2$ and color strength $=1.0$ as it provides best or second best results for all our benchmarks. It improves our baseline by $\mathbf{+2.1}$\textbf{p.p.} on Kinetics200 and UCF101, and $\mathbf{+5.0}$\textbf{p.p.} on HMDB51. It outperforms our supervised baseline by $\mathbf{+0.9}$\textbf{p.p.} on UCF101 and $\mathbf{+1.9}$\textbf{p.p.} on HMDB51.

\subsubsection{Comparison with the State of the Art}\label{Video SOTA}

\textbf{Pretraining dataset.} To compare SCE with the state of the art, we perform pretraining on Kinetics400 \citep{Kay2017} introduced in \cref{Video Ablative study}.

\textbf{Evaluation datasets.} UCF101 \citep{Soomro2012} and HMDB51 \citep{Kuehne2011} have been introduced in \cref{Video Ablative study}.

AVA (v2.2) \citep{Gu2018} is a dataset used for spatiotemporal localization of humans actions composed of 211k training videos and 57k validation videos for 60 different classes. Bounding box annotations are used as targets and we report the mean Average Precision (mAP) for evaluation.

Something-Something V2 (SSv2) \citep{Goyal2017} is a dataset composed of human-object interactions for 174 different classes. It contains 169k training and 25k validation videos.

\textbf{Pretraining implementation details.} We use the ResNet3D-18 and ResNet3D-50 network \citep{Hara2018} and more specifically the slow path of \citet{Feichtenhofer2019}. We kept the best hyperparameters from \cref{Video Ablative study} which are $\lambda=0.5$, $\tau_m=0.05$, RGB difference with probability of $0.2$, and color strength $=1.0$ on top of the $strong-\alpha$ and $strong-\beta$ augmentations. From the randomly sampled clips we specify if we keep 8 or 16 frames.

\begin{table*}
    \centering
    \footnotesize
    \begin{tabular}{l!{\vrule width -1pt}c!{\vrule width -1pt}c!{\vrule width -1pt}c!{\vrule width -1pt}c!{\vrule width -1pt}c!{\vrule width -1pt}c!{\vrule width -1pt}c!{\vrule width -1pt}c!{\vrule width -1pt}c}\toprule
        Method                 & T$_p$ & Res$_p$ & T$_e$ & Res$_e$ & Modality & Pretrain & K400            & UCF101          & HMDB51          \\ \midrule
        \multicolumn{10}{l}{\textbf{Backbone: S3D}} \\ \midrule
        CoCLR \citep{Han2020a}         & $32$  & $128^2$ & $32$  & $128^2$ & R        & K400     & -               & $87.9$          & $54.6$          \\
        \rowcolor{light-gray}
        CoCLR \citep{Han2020a}        & $32$  & $128^2$ & $32$  & $128^2$ & R+F      & K400     & -               & $90.6$          & $62.9$          \\ \midrule
        \multicolumn{10}{l}{\textbf{Backbone: S3D-G}} \\ \midrule
        SpeedNet \citep{Benaim2020} & $64$  & -       & $16$  & $224^2$ & R        & K400     & -               & $81.1$          & $48.8$          \\
        TEC \citep{Jenni2021} & $32$  & $128^2$ & $32$     & $128^2$       & R        & K400     & -               & $86.9$          & $63.5$ \\
        $\rho$BYOL ($\rho=4$) \citep{Feichtenhofer2021}    & $32$  & $224^2$ & $32$  & $256^2$ & R        & K400     & -               & $\textbf{96.3}$ & $\textbf{75.0}$ \\ \midrule
        \multicolumn{10}{l}{\textbf{Backbone: R(2+1)D-18}} \\ \midrule
        VideoMoCo \citep{Pan2021} & $32$  & $112^2$ & -     & -       & R        & K400     & -               & $78.7$          & $49.2$          \\ 
        RSPNet \citep{Chen2021}     & $16$                     & $112^2$                & $16$                        & $224^2$      & R        & K400     & -               & $81.1$          & $44.6$          \\
        TransRank \citep{Duan2022} & $16$  & $112^2$ & -     & -       & R        & K200     & -               & $87.8$          & $60.1$          \\
        \rowcolor{light-gray}
        TransRank \citep{Duan2022}  & $16$  & $112^2$ & -     & -       & R+RD     & K200     & -               & $\textbf{90.7}$ & $\textbf{64.2}$ \\
        TEC \citep{Jenni2021} & $16$  & $112^2$ & $16$     & $112^2$       & R        & K400     & -               & $88.2$          & $62.2$ \\ \midrule
        \multicolumn{10}{l}{\textbf{Backbone: ResNet3D-18}} \\ \midrule
        ST-Puzzle \citep{Kim2019} & $16$  & - & $16$     & $112^2$       & R        & K400     & -               & $65.8$          & $33.7$ \\
        3D-RotNet \citep{Jing2018}     & $16$  & $112^2$       & -  & - & R        & K400     & -               & $66.0$          & $37.1$          \\
        \rowcolor{light-gray}
        3D-RotNet \citep{Jing2018}   & $16$  & $112^2$       & -  & - & R+D        & K400     & -               & $76.7$          & $47.0$          \\
        VTHCL \citep{Yang2020}  & $8$   & $224^2$ & $8$   & $224^2$ & R        & K400     & -               & $80.6$          & $48.6$          \\
        TransRank \citep{Duan2022}  & $16$  & $112^2$ & -     & -       & R        & K200     & -               & $85.7$          & $58.1$          \\
        \rowcolor{light-gray}
        TransRank \citep{Duan2022}    & $16$  & $112^2$ & -     & -       & R+RD     & UCF101   & -               & $88.5$          & $63.0$          \\
        \rowcolor{light-gray}
        TransRank \citep{Duan2022}   & $16$  & $112^2$ & -     & -       & R+RD     & K200     & -               & $89.6$          & $63.5$          \\
        TEC \citep{Jenni2021} & $16$  & $128^2$ & $16$     & $128^2$       & R        & K400     & -               & $87.1$          & $63.6$ \\
        ProViCo \citep{Park2022}   & $16$  & $112^2$ & -     & -       & R        & K400     & -               & $87.2$          & $59.4$          \\
        $\rho$MoCo ($\rho=2$) \citep{Feichtenhofer2021} & $8$   & $224^2$ & $8$   & $256^2$ & R        & K400     & $56.2$          & $87.1$          & -               \\
        \rowcolor{light-green}
        \textbf{SCE (Ours)}   & $8$  & $224^2$ & $8$  & $256^2$ & R        & K200     & $-$          & $88.4$          & $62.0$          \\
        \rowcolor{light-green}
        \textbf{SCE (Ours)}   & $16$  & $112^2$ & $16$  & $128^2$ & R        & K400     & $56.6$          & $89.1$          & $\textbf{66.6}$          \\
        \rowcolor{light-green}
        \textbf{SCE (Ours)}  & $8$   & $224^2$ & $8$   & $256^2$ & R        & K400     & $\textbf{59.8}$ & $\textbf{90.9}$ & $65.7$ \\ \midrule
        \multicolumn{10}{l}{\textbf{Backbone: ResNet3D-50}} \\ \midrule
        VTHCL \citep{Yang2020}   & $8$   & $224^2$ & $8$   & $224^2$ & R        & K400     & $-$             & $82.1$          & $49.2$          \\
        CATE \citep{Sun2021}   & $8$   & $224^2$ & $32$   & $256^2$ & R        & K400     & $-$             & $88.4$          & $61.9$          \\
        CVRL \citep{Qian2021}  & $16$  & $224^2$ & $32$  & $256^2$ & R        & K400     & $66.1$          & $92.2$          & $66.7$          \\
        CVRL \citep{Qian2021}    & $16$  & $224^2$ & $32$  & $256^2$ & R        & K600     & $70.4$          & $93.4$          & $68.0$          \\
        \rowcolor{light-gray}
        CORP$_f$ \citep{Hu2021a}   & $16$  & $224^2$ & $32$  & $256^2$ & R+F      & K400     & $66.6$          & $93.5$          & $68.0$          \\
        ConST-CL \citep{Yuan2022}     & $16$  & $224^2$ & $32$  & $256^2$ & R        & K400     & $66.6$          & $94.8$          & $71.9$          \\
        BraVe \citep{Recasens2021}   & $16$  & $224^2$ & $32$  & $224^2$ & R        & K400     & -               & $93.7$          & $72.0$          \\
        \rowcolor{light-gray}
        BraVe \citep{Recasens2021}  & $16$  & $224^2$ & $32$  & $224^2$ & R+F      & K400     & -               & $94.7$          & $72.7$          \\
        BraVe \citep{Recasens2021}    & $16$  & $224^2$ & $32$  & $224^2$ & R        & K600     & -               & $94.1$          & $74.0$          \\
        \rowcolor{light-gray}
        BraVe \citep{Recasens2021}   & $16$  & $224^2$ & $32$  & $224^2$ & R+F      & K600     & -               & $95.1$          & $74.3$          \\
        $\rho$MoCo ($\rho=2$) \citep{Feichtenhofer2021} & $8$   & $224^2$ & $8$   & $256^2$ & R        & K400     & $65.8$          & $91.0$          & -               \\
        $\rho$MoCo ($\rho=2$) \citep{Feichtenhofer2021} & $16$  & $224^2$ & $16$  & $256^2$ & R        & K400     & $67.6$          & $93.3$          & -               \\
        $\rho$BYOL ($\rho=2$) \citep{Feichtenhofer2021}  & $8$   & $224^2$ & $8$   & $256^2$ & R        & K400     & $65.8$          & $92.7$          & -               \\
        $\rho$BYOL ($\rho=4$) \citep{Feichtenhofer2021}  & $8$   & $224^2$ & $8$   & $256^2$ & R        & K400     & $70.0$          & $94.2$          & $72.1$          \\
        $\rho$BYOL ($\rho=4$) \citep{Feichtenhofer2021}   & $8$   & $224^2$ & $16$  & $256^2$ & R        & K400     & $\mathbf{71.5}$ & $\mathbf{95.5}$ & $73.6$          \\
        \rowcolor{light-green}
        \textbf{SCE (Ours)}    & $8$   & $224^2$ & $8$   & $256^2$ & R        & K400     & $67.6$          & $94.1$          & $70.5$          \\
        \rowcolor{light-green}
        \textbf{SCE (Ours)} & $16$  & $224^2$ & $16$  & $256^2$ & R        & K400     & $69.6$          & $95.3$          & $\mathbf{74.7}$ \\
        \bottomrule
    \end{tabular}
    \caption{Performance of SCE for the linear evaluation protocol on Kinetics400 and finetuning on the three splits of UCF101 and HMDB51. \textbf{Res$_p$}, \textbf{Res$_e$} means the resolution for pretraining and evaluation. \textbf{T$_p$}, \textbf{T$_e$} means the number of frames used for pretraining and evaluation. For \textbf{Modality}, ``R" means RGB, ``F" means Optical Flow, ``RD" means RGB difference. Best viewed in color, \colorbox{light-gray}{gray rows} highlight multi-modal trainings and \colorbox{light-green}{green rows} our results. SCE obtains state of the art results on ResNet3D-18 and on the finetuning protocol for ResNet3D-50.}
    \label{tab:video-sota-action}
\end{table*}

\begin{table*}
    \centering
    \footnotesize
    \begin{tabular}{lccccc|ccc|ccc}\toprule
        \multirow{2}{*}{Method} & \multirow{2}{*}{Res$_p$} & \multirow{2}{*}{T$_p$} & \multirow{2}{*}{Res$_e$} & \multirow{2}{*}{T$_e$} & \multirow{2}{*}{Pretrain} & \multicolumn{3}{c|}{UCF101} & \multicolumn{3}{c}{HMDB51}                                                                                                             \\
                                   &                                                    &                        &                          &                        &                           & R@$1$                       & R@$5$                      & R@$10$          & R@$1$           & R@$5$           & R@$10$          \\ \midrule
        \multicolumn{12}{l}{\textbf{Backbone: ResNet3D-18}} \\ \midrule
        MemDPC \citep{Han2020b}     & $40$                     & $224^2$                & $40$                        & $224^2$                      & UCF101                       & $20.2$                      & $40.4$                     & $52.4$                         & $7.7$          & $25.7$      & $40.6$                       \\
        RSPNet \citep{Chen2021}          & $16$                     & $112^2$                & $16$                        & $224^2$                      & K400                       & $41.1$                      & $59.4$                     & $68.4$                 & -          & -      & -                         \\
        MFO \citep{Qian2021a}         & $16$                     & $112^2$                & $16$                        & $112^2$                      & K400                       & $41.5$                      & $60.6$                     & $71.2$                         & $20.7$          & $40.8$      & $55.2$           \\
        TransRank \citep{Duan2022}   & $16$                     & $112^2$                & -                        & -                      & UCF101                       & $46.5$                      & $63.7$                              & -               & $19.4$          & $45.4$          & $59.1$                         \\
        ViCC \citep{Toering2022}      & $16$                     & $128^2$                & $16$                     & $128^2$                & UCF101                       & $50.3$                      & $70.9$                     & $78.7$                    & $22.7$          & $46.2$          & $60.9$                    \\
        TransRank \citep{Duan2022}     & $16$                     & $112^2$                & -                        & -                      & K200                      & $54.0$                      & $71.8$                              & -               & $25.5$          & $52.3$          & $65.8$                         \\
        TCLR \citep{Dave2022}           & $16$                     & $112^2$                & -                        & -                      & UCF101                      & $56.2$                      & $72.2$                     & $79.0$                       & $22.8$          & $45.4$          & $57.8$                        \\ 
        TEC \citep{Jenni2021}                & $16$                     & $128^2$                & $16$                        & $128^2$                      & UCF101                       & $63.6$                      & $79.0$                     & $84.8$                  & $32.2$          & $60.3$          & $71.6$                  \\
        ProViCo \citep{Park2022}           & $16$                     & $112^2$                & -                        & -                      & UCF101                       & $63.8$                      & $75.1$                     & $84.8$                  & $35.9$          & $55.2$          & $74.3$                    \\
        ProViCo \citep{Park2022}           & $16$                     & $112^2$                & -                        & -                      & K400                      & $67.6$                      & $81.4$                     & $90.1$  & $\mathbf{40.1}$ & $60.6$          & $75.2$          \\
        \rowcolor{light-green}
        \textbf{SCE (Ours)}            & $16$                     & $112^2$                & $16$                     & $128^2$                & K400                      & $\mathbf{74.5}$                      & $\mathbf{85.9}$                     & $\mathbf{90.5}$                    & $37.8$          & $62.1$          & $73.8$            \\
        \rowcolor{light-green}
        \textbf{SCE (Ours)}           & $8$                      & $224^2$                & $8$                      & $256^2$                & K400                      & $74.4$             & $85.6$            & $90.0$                   & $\mathbf{40.1}$ & $\mathbf{63.3}$ & $\mathbf{75.4}$      \\ \midrule
        \multicolumn{12}{l}{\textbf{Backbone: ResNet3D-50}} \\ \midrule
        CATE \citep{Sun2021} & $8$                      & $224^2$                & $32$                     & $256^2$                & K400                      & $54.9$                      & $68.3$                     & $75.1$                  & $33.0$          & $56.8$          & $69.4$     \\
        \rowcolor{light-green}
        \textbf{SCE (Ours)}     & $8$                      & $224^2$                & $8$                      & $256^2$                & K400                      & $81.5$                      & $89.7$                     & $92.8$ & $43.0$          & $67.0$          & $79.0$          \\
        \rowcolor{light-green}
        \textbf{SCE (Ours)}         & $16$                     & $224^2$                & $16$                     & $256^2$                & K400                      & $\mathbf{83.9}$             & $\mathbf{92.2}$            & $\mathbf{94.9}$ & $\mathbf{45.9}$ & $\mathbf{69.9}$ & $\mathbf{80.5}$ \\
        \bottomrule
    \end{tabular}
    \caption{Performance of SCE for video retrieval on the first split of UCF101 and HMDB51. \textbf{Res$_p$}, \textbf{Res$_e$} means the resolution for pretraining and evaluation. \textbf{T$_p$}, \textbf{T$_e$} means the number of frames used for pretraining and evaluation. We report the recall R@$1$, R@$5$, R@$10$. We obtain state of the art results for ResNet3D-18 on both benchmarks and further improve our results using the larger network ResNet3D-50.}
    \label{tab:video-sota-retrieval}
\end{table*}

\textbf{Action recognition.} We compare SCE on the linear evaluation protocol on Kinetics400 and finetuning on UCF101 and HMDB51. We kept the same implementation details as in \cref{Video Ablative study}. We compare our results with the state of the art in \cref{tab:video-sota-action} on various architectures. To propose a fair comparison, we indicate for each approach the pretraining dataset, the number of frames and resolution used during pre-training as well as during evaluation. For the unknown parameters, we leave the cell empty. We compared with some approaches that used the other visual modalities Optical Flow and RGB difference and the different convolutional backbones S3D \citep{Zhang2018} and R(2+1)D-18 \citep{Tran2018}. 

On ResNet3D-18 even when comparing with methods using several modalities, by using $8 \times 224^2$ frames we obtain state-of-the-art results on the three benchmarks with $\mathbf{59.8\%}$ accuracy on Kinetics400, $\mathbf{90.9\%}$ on UCF101, $\mathbf{65.7\%}$ on HMDB51. Using $16 \times 112^2$ frames, which is commonly used with this network, improved by $+0.9$p.p on HMDB51 and decreased by $-3.2$p.p on kinetics400 and $-1.8$ on UCF101 and keep state of the art results on all benchmarks, except on UCF101 with $-0.5$p.p compared with \citet{Duan2022} using RGB and RGB difference modalities.

On ResNet3D-50, we obtain state-of-the-art results using $16 \times 224^2$ frames on HMDB51 with $\mathbf{74.7\%}$ accuracy even when comparing with methods using several modalities. On UCF101, with $\mathbf{95.3}$\% SCE is on par with the state of the art, $-0.2$p.p. than \citep{Feichtenhofer2021}, but on Kinetics400 $-1.9$p.p for $\mathbf{69.6\%}$. We have the same computational budget as they use $4$ views for pretraining. Using 8 frames decreased performance by $-2.0$p.p., $-1.2$p.p. and $-4.2$p.p on Kinetics400,UCF101 and HMDB51. It maintains results that outperform on the three benchmarks $\rho$MoCo and $\rho$BYOL with 2 views. It suggests that SCE is more efficient with fewer resources than these methods. By comparing our best with approaches on the S3D backbone that better fit smaller datasets, SCE has slightly lower performance than the state of the art: $-1.0$p.p. on UCF101 and $-0.3$p.p. on HMDB51.

\begin{table*}
    \centering
    \footnotesize
    \begin{tabular}{lcc|c|ccc}\toprule
                                &       &    & \emph{Linear protocol}                     & \multicolumn{3}{c}{\emph{Finetuning accuracy}}                                                                                             \\
        Method                  & views & T  & K400                                       & UCF101                                         & AVA (mAP)                                   & SSv2                                        \\ \midrule
        Supervised              & 1     & 8  & $\textbf{74.7}$                            & $94.8$                                         & $22.2$                                      & $52.8$                                      \\ \midrule
        $\rho$SimCLR ($\rho=3$) & 3     & 8  & $62.0$ $\color{BrickRed}(-12.7)$           & $87.9$ $\color{BrickRed}(-6.9)$                & $17.6$ $\color{BrickRed}(-4.6)$             & $52.0$ $\color{BrickRed}(-0.8)$             \\
        $\rho$SwAV ($\rho=3$)   & 3     & 8  & $62.7$ $\color{BrickRed}(-12.0)$           & $89.4$ $\color{BrickRed}(-5.4)$                & $18.2$ $\color{BrickRed}(-4.0)$             & $51.7$ $\color{BrickRed}(-1.1)$             \\
        $\rho$BYOL ($\rho=3$)   & 3     & 8  & $68.3$ $\color{BrickRed}(-6.4)\phantom{0}$ & $93.8$ $\color{BrickRed}(-1.0)$                & $\textbf{23.4}$ $\color{ForestGreen}(+1.2)$ & $55.8$ $\color{ForestGreen}(+3.0)$          \\
        $\rho$MoCo ($\rho=3$)   & 3     & 8  & $67.3$ $\color{BrickRed}(-7.4)\phantom{0}$ & $92.8$ $\color{BrickRed}(-2.0)$                & $20.3$ $\color{BrickRed}(-1.9)$             & $54.4$ $\color{ForestGreen}(+1.8)$          \\ \midrule
        SCE (Ours)              & 2     & 8  & $67.6$ $\color{BrickRed}(-7.1)\phantom{0}$ & $94.1$ $\color{BrickRed}(-0.7)$                & $20.3$ $\color{BrickRed}(-1.9)$             & $53.9$ $\color{ForestGreen}(+1.1)$          \\
        SCE (Ours)              & 2     & 16 & $69.6$ $\color{BrickRed}(-5.1)\phantom{0}$ & $\textbf{95.5}$ $\color{ForestGreen}(+0.7)$    & $21.6$ $\color{BrickRed}(-0.6)$             & $\textbf{57.2}$ $\color{ForestGreen}(+4.4)$ \\
        \bottomrule
    \end{tabular}
    \caption{Performance of SCE in comparison with \citet{Feichtenhofer2021} for linear evaluation on Kinetics400 and finetuning on the first split of UCF101, AVA and SSv2. SCE is on par with $\rho$MoCo for fewer views. Increasing the number of frames outperforms $\rho$BYOL on Kinetics400, UCF101 and SSv2.}
    \label{tab:video-sota-generalization}
\end{table*}

\textbf{Video retrieval.} We performed video retrieval on our pretrained backbones on the first split of UCF101 and HMDB51. To perform this task, we extract from the training and testing splits the features using the 30-crops procedure as for action recognition, detailed in Appendix D.3. We query for each video in the testing split the $N$ nearest neighbors ($N\in\{1,5,10\}$) in the training split using cosine similarities. We report the recall R@$N$ for the different $N$ in \cref{tab:video-sota-retrieval}.

We compare our results with the state of the art on ResNet3D-18. Our proposed SCE with $16 \times 112^2$ frames is Top-1 on UCF101 with $\mathbf{74.5\%}$, $\mathbf{85.6\%}$ and $\mathbf{90.5\%}$ for R@$1$, R@$5$ and R@$10$. Using $8 \times 224^2$ frames slightly decreases results that are still state of the art. On HMDB51, SCE with $8 \times 224^2$ frames outperforms the state of the art with $\mathbf{40.1\%}$, $\mathbf{63.3\%}$ and $\mathbf{75.4\%}$ for R@$1$, R@$5$ and R@$10$. Using $16 \times 112^2$ frames decreased results that are competitive with the previous state of the art approach \citep{Park2022} for $-2.3$p.p., $+1.5$p.p. and $-1.4$p.p. on R@$1$, R@$5$ and R@$10$.

We provide results using the larger architecture ResNet3d-50 which increases our performance on both benchmarks and outperforms the state of the art on all metrics to reach $\mathbf{83.9\%}$, $\mathbf{92.2\%}$ and $\mathbf{94.9\%}$ for R@$1$, R@$5$ and R@$10$ on UCF101 as well as $\mathbf{45.9\%}$, $\mathbf{69.9\%}$ and $\mathbf{80.5\%}$ for R@$1$, R@$5$ and R@$10$ on HMDB51. Our soft contrastive learning approach makes our representation learn features that cluster similar instances even for generalization.

\textbf{Generalization to downstream tasks.} We follow the protocol introduced by \citet{Feichtenhofer2021} to compare the generalization of our ResNet3d-50 backbone on Kinetics400, UCF101, AVA and SSv2 with $\rho$SimCLR, $\rho$SwAV, $\rho$BYOL, $\rho$MoCo and supervised learning in \cref{tab:video-sota-generalization}. To ensure a fair comparison, we provide the number of views used by each method and the number of frames per view for pretraining and evaluation. 

For 2 views and 8 frames, SCE is on par with $\rho$MoCo with 3 views on Kinetics400, AVA and SSv2 but is worst than $\rho$BYOL especially on AVA. For UCF101, results are better than $\rho$MoCo and on par with $\rho$BYOL. These results indicate that our approach proves more effective than contrastive learning as it reaches similar results than $\rho$MoCo using one less view. Using 16 frames, SCE outperforms all approaches, including supervised training, on UCF101 and SSv2 but performs worse on AVA than $\rho$Byol and supervised training. This study shows that SCE can generalize to various video downstream tasks which is a criteria of a good learned representation.

\section{Conclusion}\label{Conclusion}
In this paper we introduced a self-supervised soft contrastive learning approach called Similarity Contrastive Estimation (SCE). It contrasts pairs of asymmetrical augmented views with other instances while maintaining relations among instances. SCE leverages contrastive learning and relational learning and improves the performance over optimizing only one aspect. We showed that it is competitive with the state of the art on the linear evaluation protocol on ImageNet, on video representation learning and to generalize to several image and video downstream tasks. We proposed a simple but effective initial estimation of the true distribution of similarity among instances. An interesting perspective would be to propose a finer estimation of this distribution.

\backmatter


\bmhead*{Acknowledgments}
This publication was made possible by the use of the Factory-AI supercomputer, financially supported by the Ile-de-France Regional Council, and the HPC resources of IDRIS under the allocation 2022-AD011013575 made by GENCI.

\bmhead*{Declarations}
The authors have no relevant financial or non-financial interests to disclose.

\bibliography{arxiv}
\clearpage

\begin{appendix}

\clearpage
\section{Pseudo-Code of SCE}\label{sec-code}

\begin{minipage}{\textwidth}
\centering
\begin{lstlisting}[language=Python, xleftmargin=0.5cm, breaklines=false, numbers=left, caption=Pseudo-Code of SCE in Pytorch style, basicstyle=\scriptsize]
# dataloader: loader of batches
# bsz: batch size
# epochs: number of epochs
# T1: weak distribution of data augmentations
# T2: strong distribution of data augmentations
# f_s, g_s, h_s: online encoder, projector, and optional predictor
# f_t, g_t: momentum encoder and projector 
# queue: memory buffer
# tau: online temperature
# tau_m: momentum temperature
# lambda_: coefficient between contrastive and relational aspects
# symmetry_loss: if True, symmetries the loss

def sce_loss(z1, z2):
    sim2_pos = zeros(bsz)
    sim2_neg = einsum("nc,kc->nk", z2, queue)
    sim2 = cat([sim2_pos, sim2_neg]) / tau_m
    s2 = softmax(sim2)
    w2 = lambda_ * one_hot(sim2_pos, bsz+1) + (1 - lambda_) * s
    
    sim1_pos = einsum("nc,nc->n", z1, z2)
    sim1_neg = einsum("nc,kc->nk", z1, queue)
    sim1 = cat([sim1_pos, sim1_neg]) / tau
    p1 = softmax(sim1)
    
    loss = cross_entropy(p1, w2)
    return loss
    
for i in range(epochs):
    for x in dataloader:
        x1, x2 = T1(x), T2(x)
        
        z1_s, z2_t = h_s(g_s(f_s(x1))), g_t(f_t(x2))
        z2_t = stop_grad(z2_t)
        
        loss = sce_loss(z1_s, z2_t)
        if symmetry_loss:
            z1_t, z2_s = g_t(f_t(x1)), h_s(g_s(f_s(x2)))
            z1_t = stop_grad(z1_t)
            loss += sce_loss(z2_s, z1_t)
            loss /= 2
        loss.backward()
        
        update(f_s.params)
        update(g_s.params)
        update(h_s.params)
        momentum_update(f_t.params, f_s.params)
        momentum_update(g_t.params, g_s.params)
        
        fifo_update(queue, z2_t)
        if symmetry_loss:
            fifo_update(queue, z1_t)
\end{lstlisting}
\end{minipage}

\clearpage
\section{Proof Proposition 1.}\label{sec-imple-proof}

\begin{prop*}
$L_{SCE}$ defined as
\footnotesize
$$L_{SCE} = - \frac{1}{N} \sum_{i=1}^N\sum_{k=1}^N w^2_{ik} \log\left(p^1_{ik}\right),$$
\normalsize
can be written as:
\footnotesize
\begin{equation*}
  L_{SCE} = \lambda \cdot L_{InfoNCE} + \mu \cdot L_{ReSSL} +  \eta \cdot L_{ceil},
\end{equation*}
\normalsize
with \footnotesize$\mu = \eta = 1 - \lambda$ \normalsize and
\footnotesize
\begin{equation*}
  L_{Ceil} = - \frac{1}{N} \sum_{i=1}^{N}\log\left(\frac{\sum_{j=1}^{N}\mathbbm{1}_{i \neq j} \cdot \exp(\mathbf{z^1_i} \cdot \mathbf{z^2_j} / \tau)}{\sum_{j=1}^{N} \exp(\mathbf{z^1_i} \cdot \mathbf{z^2_j} / \tau)}\right).
\end{equation*}
\end{prop*}

\begin{proof}
\normalsize
Recall that:
\footnotesize
\begin{equation*}
\setlength{\abovedisplayskip}{12pt}
\setlength{\belowdisplayskip}{12pt}
\begin{aligned}
    p^1_{ik} &= \frac{\exp(\mathbf{z^1_i} \cdot \mathbf{z^2_k} / \tau)}{\sum_{j=1}^{N}\exp(\mathbf{z^1_i} \cdot \mathbf{z^2_j} / \tau)}, \\
    s^2_{ik} &= \frac{\mathbbm{1}_{i \neq k} \cdot \exp(\mathbf{z^2_i} \cdot \mathbf{z^2_k} / \tau_m)}{\sum_{j=1}^{N}\mathbbm{1}_{i \neq j} \cdot \exp(\mathbf{z^2_i} \cdot \mathbf{z^2_j} / \tau_m)}, \\
    w^2_{ik} &= \lambda \cdot \mathbbm{1}_{i=k} + (1 - \lambda) \cdot s^2_{ik}.
\end{aligned}
\end{equation*}
\normalsize
We decompose the second loss over $k$ in the definition of $L_{SCE}$ to make the proof:
\footnotesize
\begin{equation*}
\setlength{\abovedisplayskip}{12pt}
\setlength{\belowdisplayskip}{12pt}
\begin{aligned}
    L_{SCE} = & - \frac{1}{N} \sum_{i=1}^N\sum_{k=1}^N w^2_{ik} \log\left(p^1_{ik}\right) \\
            = & - \frac{1}{N} \sum_{i=1}^N \left[ w^2_{ii} \log\left(p^1_{ii}\right) + \sum_{\substack{k=1 \\ k\neq i}}^N w^2_{ik} log\left(p^1_{ik}\right)\right] \\
            = & \underbrace{- \frac{1}{N} \sum_{i=1}^N w^2_{ii} \left(p^1_{ii}\right)}_{(1)} \underbrace{- \frac{1}{N} \sum_{i=1}^N\sum_{\substack{k=1 \\ k\neq i}}^N w^2_{ik} \log\left(p^1_{ik}\right)}_{(2)}. \\
\end{aligned}
\end{equation*}
\normalsize
First we rewrite $(1)$ to retrieve the $L_{InfoNCE}$ loss.
\footnotesize
\begin{equation*}
\setlength{\abovedisplayskip}{12pt}
\setlength{\belowdisplayskip}{12pt}
\begin{aligned}
    (1) = & - \frac{1}{N} \sum_{i=1}^N w^2_{ii} log\left(p^1_{ii}\right) \\
        = & - \frac{1}{N} \sum_{i=1}^N \lambda \cdot \log\left(p^1_{ii}\right) \\
        = & - \lambda \cdot \frac{1}{N} \sum_{i=1}^N \log\left(\frac{\exp(\mathbf{z^1_i} \cdot \mathbf{z^2_i} / \tau)}{\sum_{j=1}^{N}\exp(\mathbf{z^1_i} \cdot \mathbf{z^2_j} / \tau)}\right) \\
        = & \lambda \cdot L_{InfoNCE}.
\end{aligned}
\end{equation*}
\normalsize
Now we rewrite $(2)$ to retrieve the $L_{ReSSL}$ and $L_{Ceil}$ losses.
\footnotesize
\begin{equation*}
\setlength{\abovedisplayskip}{12pt}
\setlength{\belowdisplayskip}{0pt}
\begin{aligned}
    (2) = & - \frac{1}{N} \sum_{i=1}^N\sum_{\substack{k=1 \\ k\neq i}}^N w^2_{ik} \log\left(p^1_{ik}\right) \\
        = & - \frac{1}{N} \sum_{i=1}^N\sum_{\substack{k=1 \\ k\neq i}}^N (1 - \lambda) \cdot s^2_{ik} \cdot \log\left(p^1_{ik}\right) \\
        = & - (1 - \lambda) \cdot \frac{1}{N} \sum_{i=1}^N\sum_{k=1}^N s^2_{ik} \cdot \log\left(p^1_{ik}\right) \\
        = & - (1 - \lambda) \cdot \frac{1}{N} \sum_{i=1}^N\sum_{k=1}^N \left[ s^2_{ik} \cdot \log\left(\frac{\exp(\mathbf{z^1_i} \cdot \mathbf{z^2_k} / \tau)}{\sum_{j=1}^{N}\exp(\mathbf{z^1_i} \cdot \mathbf{z^2_j} / \tau)}\right) \right]\\
\end{aligned}
\end{equation*}
\begin{equation*}
\setlength{\abovedisplayskip}{0pt}
\setlength{\belowdisplayskip}{12pt}
\begin{aligned}
        &\phantom{(2)} = - (1 - \lambda) \cdot \frac{1}{N} \sum_{i=1}^N\sum_{k=1}^N \Bigg[ s^2_{ik} \cdot \Bigg( \\
          & \phantom{ = - (1 - \lambda) \cdot} \log\left(\exp(\mathbf{z^1_i} \cdot \mathbf{z^2_k} / \tau)\right) - \log\left(\sum_{j=1}^{N}\exp(\mathbf{z^1_i} \cdot \mathbf{z^2_j} / \tau)\right)\Bigg)\Bigg] \\
        &\phantom{(2)} = - (1 - \lambda) \cdot \frac{1}{N} \sum_{i=1}^N\sum_{k=1}^N \Bigg[ s^2_{ik} \cdot \Bigg( \\
        & \phantom{ = - (1 - \lambda) \cdot} \log\left(\exp(\mathbf{z^1_i} \cdot \mathbf{z^2_k} / \tau)\right) - \log\left(\sum_{j=1}^{N}\exp(\mathbf{z^1_i} \cdot \mathbf{z^2_j} / \tau)\right) + \\
          & \phantom{ = - (1 - \lambda) \cdot} \log\left(\sum_{j=1}^{N} \mathbbm{1}_{i \neq j} \cdot \exp(\mathbf{z^1_i} \cdot \mathbf{z^2_j} / \tau)\right) - \\
          & \phantom{ = - (1 - \lambda) \cdot} \log\left(\sum_{j=1}^{N} \mathbbm{1}_{i \neq j} \cdot \exp(\mathbf{z^1_i} \cdot \mathbf{z^2_j} / \tau)\right) \Bigg)\Bigg] \\
        &\phantom{(2)} =  - (1 - \lambda) \cdot \frac{1}{N} \sum_{i=1}^N\sum_{k=1}^N \Bigg[ s^2_{ik} \cdot \Bigg( \\
          & \phantom{ = - (1 - \lambda)} \log\left(\frac{\exp(\mathbf{z^1_i} \cdot \mathbf{z^2_k} / \tau)}{\sum_{j=1}^{N} \mathbbm{1}_{i \neq j} \cdot \exp(\mathbf{z^1_i} \cdot \mathbf{z^2_j} / \tau)}\right) + \\
          & \phantom{ = - (1 - \lambda)} \log\left(\frac{\sum_{j=1}^{N} \mathbbm{1}_{i \neq j} \cdot \exp(\mathbf{z^1_i} \cdot \mathbf{z^2_j} / \tau)}{\sum_{j=1}^{N}\exp(\mathbf{z^1_i} \cdot \mathbf{z^2_j} / \tau)}\right)\Bigg)\Bigg] \\
        &\phantom{(2)} = - (1 - \lambda) \cdot \frac{1}{N} \Bigg( \\
        & \sum_{i=1}^N\sum_{k=1}^N \Bigg[ s^2_{ik} \cdot \log\left(\frac{\exp(\mathbf{z^1_i} \cdot \mathbf{z^2_k} / \tau)}{\sum_{j=1}^{N} \mathbbm{1}_{i \neq j} \cdot \exp(\mathbf{z^1_i} \cdot \mathbf{z^2_j} / \tau)}\right)\Bigg] + \\
        & \sum_{i=1}^N\sum_{k=1}^N \Bigg[ s^2_{ik} \cdot \log\left(\frac{\sum_{j=1}^{N} \mathbbm{1}_{i \neq j} \cdot \exp(\mathbf{z^1_i} \cdot \mathbf{z^2_j} / \tau)}{\sum_{j=1}^{N}\exp(\mathbf{z^1_i} \cdot \mathbf{z^2_j} / \tau)}\right)\Bigg]\Bigg). \\
\end{aligned}
\end{equation*}
\normalsize
Because $s^2_{ii} = 0$ and $\mathbf{s^2_i}$ is a probability distribution, we have:
\begin{equation*}
\footnotesize
\begin{aligned}
        (2) = & - (1 - \lambda) \cdot \\
        & \frac{1}{N} \sum_{i=1}^N\sum_{\substack{k=1 \\ k\neq i}}^N \Bigg[ s^2_{ik} \cdot \log\left(\frac{\mathbbm{1}_{i \neq k} \cdot \exp(\mathbf{z^1_i} \cdot \mathbf{z^2_k} / \tau)}{\sum_{j=1}^{N} \mathbbm{1}_{i \neq j} \cdot \exp(\mathbf{z^1_i} \cdot \mathbf{z^2_j} / \tau)}\right)\Bigg] - \\
        & (1 - \lambda) \cdot \frac{1}{N} \sum_{i=1}^N \Bigg[ \log\left(\frac{\sum_{j=1}^{N} \mathbbm{1}_{i \neq j} \cdot \exp(\mathbf{z^1_i} \cdot \mathbf{z^2_j} / \tau)}{\sum_{j=1}^{N}\exp(\mathbf{z^1_i} \cdot \mathbf{z^2_j} / \tau)}\right) \Bigg] \\
        = & (1 - \lambda) \cdot L_{ReSSL} + (1 - \lambda) \cdot L_{Ceil}.
\end{aligned}
\end{equation*}
\normalsize
\end{proof}

\begin{table*}[t]
    \small
    \centering
    \begin{tabular}{cc|cccc|cccccc} \toprule
        \multirow{2}{*}{Backbone} & \multirow{2}{*}{Dataset} & \multicolumn{4}{c|}{Projector} & \multirow{2}{*}{Input} & \multirow{2}{*}{Buffer} & \multirow{2}{*}{ema} & \multirow{2}{*}{LR} & \multirow{2}{*}{Batch} & \multirow{2}{*}{WD}\\
        & & Layers & Hid dim & Out dim & BN & & & & &\\ \midrule
        R-18 & CIFAR & 2 & 512 & 128 & hid & $32^2$ & $4,096$ & 0.900 & 0.06 & $256$ & $5e^{-4}$ \\ 
        R-18 & STL10 & 2 & 512 &  128 & hid & $96^2$ & $16,384$ & 0.996 & 0.06 & $256$ & $5e^{-4}$ \\ 
        R-18 & Tiny-IN & 2 & 512 & 128 & hid & $64^2$ & $16,384$ & 0.996 & 0.06 & $256$ & $5e^{-4}$ \\ 
        R-50 & IN100 & 2 & 4096 & 256 & no & $224^2$ & $65,536$ & 0.996 & 0.3 & $512$ & $1e^{-4}$ \\ 
        R-50 & IN1k & 3 & 2048 & 256 & all & $224^2$ & $65,536$ & 0.996 & 0.5 & $512$ & $1e^{-4}$ \\ 
        \bottomrule  
    \end{tabular}
    \caption{Architecture and hyperparameters used for pretraining on the different datasets. \textbf{LR} stands for the initial learning rate, \textbf{WD} for weight decay, \textbf{BN} for batch normalization \citep{Ioffe2015}, \textbf{Hid} for hidden, \textbf{Dim} for dimension, \textbf{ema} for the initial momentum value used to update the momentum branch. For \textbf{BN}: ``no" means no batch normalization is used in the projector, ``hid" means batch normalization after each hidden layer, ``all" means batch normalization after the hidden layer and the output layer.}
    \label{tab:specifc-hyper-pretrain}
\end{table*}

\section{Classes to construct ImageNet100}\label{sec-classes}
To build the ImageNet100 dataset, we used the classes shared by the CMC \citep{Tian2020} authors in the supplementary material of their publication. We also share these classes in \cref{tab:classes}.

\begin{table}[h]
\centering
\begin{tabular}{cccc}\toprule
\multicolumn{4}{c}{100 selected classes from ImageNet} \\ \midrule
n02869837 & n01749939 & n02488291 & n02107142 \\
n13037406 & n02091831 & n04517823 & n04589890 \\
n03062245 & n01773797 & n01735189 & n07831146 \\
n07753275 & n03085013 & n04485082 & n02105505 \\
n01983481 & n02788148 & n03530642 & n04435653 \\
n02086910 & n02859443 & n13040303 & n03594734 \\
n02085620 & n02099849 & n01558993 & n04493381 \\
n02109047 & n04111531 & n02877765 & n04429376 \\
n02009229 & n01978455 & n02106550 & n01820546 \\
n01692333 & n07714571 & n02974003 & n02114855 \\
n03785016 & n03764736 & n03775546 & n02087046 \\
n07836838 & n04099969 & n04592741 & n03891251 \\
n02701002 & n03379051 & n02259212 & n07715103 \\
n03947888 & n04026417 & n02326432 & n03637318 \\
n01980166 & n02113799 & n02086240 & n03903868 \\
n02483362 & n04127249 & n02089973 & n03017168 \\
n02093428 & n02804414 & n02396427 & n04418357 \\
n02172182 & n01729322 & n02113978 & n03787032 \\
n02089867 & n02119022 & n03777754 & n04238763 \\
n02231487 & n03032252 & n02138441 & n02104029 \\
n03837869 & n03494278 & n04136333 & n03794056 \\
n03492542 & n02018207 & n04067472 & n03930630 \\
n03584829 & n02123045 & n04229816 & n02100583 \\
n03642806 & n04336792 & n03259280 & n02116738 \\
n02108089 & n03424325 & n01855672 & n02090622 \\
\bottomrule
\end{tabular}
\caption{The 100 classes selected from ImageNet to construct ImageNet100.}
\label{tab:classes}
\end{table}

\section{Implementation details}\label{sec-imple}
\subsection{Ablation study and baseline comparison for images}\label{subsec-imple-images}

\textbf{Pretraining Implementation details.} We use the ResNet-50 \citep{He2016} encoder for large datasets and ResNet-18 for small and medium datasets with changes detailed below. We pretrain the models for 200 epochs. We apply by default \emph{strong} and \emph{weak} data augmentations, defined in Tab. 1 in the main paper, with the scaling range for the random resized crop set to $(0.2, 1.0)$. Specific hyperparameters for each dataset for the projector construction, the size of the input, the size of the memory buffer, the initial momentum value, the initial learning rate, the batch size and the weight decay applied can be found in \cref{tab:specifc-hyper-pretrain}. We use the SGD optimizer \citep{Sutskever2013} with a momentum of 0.9. A linear warmup is applied during 5 epochs to reach the initial learning rate. The learning rate is scaled using the linear scaling rule and follows the cosine decay scheduler without restart \citep{Loshchilov2016}. The momentum value to update the target branch follows a cosine strategy from its initial value to reach 1 at the end of training. We do not symmetrize the loss by default.

\textbf{Architecture change for small and medium datasets.} Because the images are smaller, and ResNet is suitable for larger images, typically $224 \times 224$, we follow guidance from SimCLR \citep{Chen2020b} and replace the first $7 \times 7$ Conv of stride $2$ with a $3\times3$ Conv of stride $1$. We also remove the first pooling layer.

\textbf{Evaluation protocol.} To evaluate our pretrained encoders, we train a linear classifier following \citep{Chen2020a, Zheng2021}. We train for 100 epochs on top of the frozen pretrained encoder using an SGD optimizer with an initial learning rate of $30$ without weight decay and a momentum of 0.9. A scheduler is applied to the learning rate that is decayed by a factor of $0.1$ at 60 and 80 epochs. The data augmentations for the different datasets are:
\begin{itemize}[noitemsep,topsep=0pt]
    \item \textbf{training set for large datasets}: random resized crop to resolution $224 \times 224$ with the scaling range set to $(0.08, 1.0)$ and a random horizontal flip with a probability of $0.5$.
    \item \textbf{training set for small and medium datasets}: random resized crop to the dataset resolution with a padding of $4$ for small datasets and the scaling range set to $(0.08, 1.0)$. Also, a random horizontal flip with a probability of $0.5$ is applied.
    \item \textbf{validation set for large datasets}: resize to resolution $256 \times 256$ and center crop to resolution $224 \times 224$.
     \item \textbf{validation set for small and medium datasets}: resize to the dataset resolution.
\end{itemize}

\subsection{Imagenet study}\label{subsec-imple-imagenet}

\textbf{Pretraining implementation details.} We use the ResNet-50 \citep{He2016} encoder and apply \emph{strong-$\alpha$} and \emph{strong-$\beta$} augmentations, defined in Tab. 1 in the main paper, with the scaling range for the random resized crop set to $(0.2, 1.0)$. The batch size is set to 4096 and the memory buffer to 65,536. We follow the same training hyperparameters as \citep{Chen2021b} for the architecture. Specifically, we use the same projector and predictor, the LARS optimizer \citep{Ginsburg2018} with a weight decay of $1.5\cdot10^{-6}$ for 1000 epochs of training and $10^{-6}$ for fewer epochs. Bias and batch normalization \citep{Ioffe2015} parameters are excluded. The initial learning rate is $0.5$ for 100 epochs and $0.3$ for more epochs. It is linearly scaled for 10 epochs and it follows the cosine annealed scheduler. The momentum value follows a cosine scheduler from 0.996 for 1000 epochs, 0.99 for fewer epochs, to reach 1 at the end of training. The loss is symmetrized. For SCE specific hyperparameters, we keep the best from ablation study: $\lambda = 0.5$, $\tau = 0.1$ and $\tau_m = 0.07$.

\textbf{Multi-crop setting.} We follow \citet{Hu2021} and sample 6 different views. The first two views are global views as without multi-crop, meaning resolution of $224 \times 224$ and the scaling range for random resized crop set to $(0.2, 1.0)$. The 4 local crops have a resolution of $192 \times 192$, $160 \times 160$, $128 \times 128$, $96 \times 96$ and scaling range ($0.172$, $0.86$), ($0.143$, $0.715$), ($0.114$, $0.571$), ($0.086$, $0.429$) on which we apply the \emph{strong-$\gamma$} data augmentation defined in Tab. 1 in the main paper.

\textbf{Evaluation protocol.} We follow the protocol defined by \citep{Chen2021b}. Specifically, we train a linear classifier for 90 epochs on top of the frozen encoder with a batch size of 1024 and a SGD optimizer with a momentum of 0.9 and without weight decay. The initial learning rate is $0.1$ and scaled using the linear scaling rule and follows the cosine decay scheduler without restart \citep{Loshchilov2016}. The data augmentations applied are:
\begin{itemize}[noitemsep,topsep=0pt]
    \item \textbf{training set}: random resized crop to resolution $224 \times 224$ with the scaling range set to $(0.08, 1.0)$ and a random horizontal flip with a probability of $0.5$.
    \item \textbf{validation set}: resize to resolution $256 \times 256$ and center crop to resolution $224 \times 224$.
\end{itemize} 

\begin{table}[]
    \centering
    \small
    \begin{tabular}{c|c|c} \toprule
    stage & ResNet3d-18 & ResNet3D-50 \\ \midrule
    \multirow{2}{*}{conv1} & $1\times 7^2, 64$ & $1\times 7^2, 64$ \\
                           & stride $1, 2^2$ & stride $1, 2^2$ \\ \midrule  
    \multirow{2}{*}{pool1} & $1\times 3^2, max$ & $1\times 3^2, max$ \\
                           & stride $1, 2^2$ & stride $1, 2^2$ \\ \midrule
    res$_2$ &  $\begin{bmatrix} 1 \times 3^2, 64 \\ 1 \times 3^2, 64 \end{bmatrix} \times 2$ & $\begin{bmatrix} 1 \times 1^2, 64\phantom{0} \\ 1 \times 3^2, 64\phantom{0} \\ 1 \times 1^2, 256 \end{bmatrix} \times 3$ \\ \midrule  
    res$_3$ &  $\begin{bmatrix} 1 \times 3^2, 128 \\ 1 \times 3^2, 128 \end{bmatrix} \times 2$ & $\begin{bmatrix} 1 \times 1^2, 128 \\ 1 \times 3^2, 128 \\ 1 \times 1^2, 512 \end{bmatrix} \times 4$ \\ \midrule  
    res$_4$ &  $\begin{bmatrix} 3 \times 3^2, 256 \\ 1 \times 3^2, 256 \end{bmatrix} \times 2$ & $\begin{bmatrix} 3 \times 1^2, 256\phantom{0} \\ 1 \times 3^2, 256\phantom{0} \\ 1 \times 1^2, 1024 \end{bmatrix} \times 6$ \\ \midrule  
    res$_5$ &  $\begin{bmatrix} 3 \times 3^2, 512 \\ 1 \times 3^2, 512 \end{bmatrix} \times 2$ & $\begin{bmatrix} 3 \times 1^2, 512\phantom{0} \\ 1 \times 3^2, 512\phantom{0} \\ 1 \times 1^2, 2048 \end{bmatrix} \times 3$ \\ \midrule
    pool & global average & global average \\ \bottomrule
    \end{tabular}
    \caption{ResNet3D-18 and ResNet3D-50 networks.}
    \label{tab:r3d}
\end{table}

\begin{table*}[t]
    \small
    \centering
    \begin{tabular}{cc|cccc|cccc|c} \toprule
        \multirow{2}{*}{Backbone} & \multirow{2}{*}{Dataset} & \multicolumn{4}{c|}{Projector} & \multicolumn{4}{c|}{Predictor} & \multirow{2}{*}{Buffer}\\
        & & Layers & Hid dim & Out dim & BN & Layers & Hid dim & Out dim & BN &\\ \midrule
        ResNet3D-18 & K200 & 3 & 1024 & 256 & all & 2 & 1024 & 256 & hid & 32768 \\ 
        ResNet3D-18 & K400 & 3 & 1024 & 256 & all & 2 & 1024 & 256 & hid & 65536 \\ 
        ResNet3D-50 & K400 & 3 & 4096 & 256 & all & 2 & 4096 & 256 & hid & 65536 \\ 
        \bottomrule  
    \end{tabular}
    \caption{Architecture and hyperparameters used for video pretraining. \textbf{BN} stands for for Batch Normalization, \textbf{Hid} for hidden, \textbf{Dim} for dimension. For \textbf{BN}: "hid" means batch normalization after each hidden layer, "all" means batch normalization after the hidden layer and the output layer.}
    \label{tab:specific-hyper-pretrain-video}
\end{table*}

\begin{table*}
    \centering
    \footnotesize
    \begin{tabular}{cccccccccc} \toprule
        Dataset & $\tau$ & $\tau_m = 0.03$ & $\tau_m = 0.04$ & $\tau_m = 0.05$ & $\tau_m = 0.06$ & $\tau_m = 0.07$ & $\tau_m = 0.08$ & $\tau_m = 0.09$ & $\tau_m = 0.1$  \\ \midrule
        CIFAR10 & 0.1 & 89.93 & 90.03 & 90.06 & 90.20 & 90.16 & 90.06 & 89.67 & 88.97 \\
        CIFAR10 & 0.2 & 89.98 & 90.12 & 90.12 & 90.05 & 90.13 & 90.09 & 90.22 & \textbf{90.34} \\ \midrule
        CIFAR100 & 0.1 & 64.49 & 64.90 & 65.19 & 65.33 & 65.27 & \textbf{65.45} & 64.89 & 63.87 \\ 
        CIFAR100 & 0.2 & 63.71 & 63.74 & 63.89 & 64.05 & 64.24 & 64.23 & 64.10 & 64.30 \\ \midrule
        STL10 & 0.1 & 89.34 & \textbf{89.94} & 89.87 & 89.84 & 89.72 & 89.52 & 88.99 & 88.41 \\
        STL10 & 0.2 & 88.4 & 88.23 & 88.4 & 88.35 & 87.54 & 88.32 & 88.80 & 88.59 \\ \midrule
        Tiny-IN & 0.1 & 50.23 & 51.12 & 51.41 & 51.66 &  \textbf{51.90} & 51.58 & 51.37 & 50.46 \\
        Tiny-IN & 0.2 & 48.56 & 48.85 & 48.35 & 48.98 & 49.06 & 49.15 & 49.66 & 49.64 \\ \bottomrule
    \end{tabular}
    \caption{Effect of varying the temperature parameters $\tau_m$ and $\tau$ on the Top-1 accuracy.}
    \label{tab:temperature-small}
\end{table*}

\subsection{Video study}\label{subsec-imple-video}
\textbf{Pretraining implementation details.} We used the ResNet3D-18 and ResNet3D-50 networks \citep{Hara2018} following the Slow path of \citet{Feichtenhofer2019}. The exact architecture details can be found in \cref{tab:r3d}. We kept the siamese architecture used for ImageNet in Sec. 4.1.3 and depending on the backbone and pretraining dataset, the projector and predictor architectures as well as the memory buffer size vary and are referenced in \cref{tab:specific-hyper-pretrain-video}. The LARS optimizer with a weight decay of $1.10^{-6}$, batch normalization and bias parameters excluded, for 200 epoch of training is used. The learning rate follows a linear warmup until it reaches an initial value of $2.4$ and then follows a cosine annealed scheduler. The initial learning rate is scaled following the linear scaling rule and the batch size is set to $512$. The momentum value follows a cosine scheduler from $0.99$ to $1$ and the loss is symmetrized. To sample and crop different views from a video, we follow \citet{Feichtenhofer2021} and sample randomly different clips from the video that lasts $2.56$ seconds. For Kinetics it corresponds to 64 frames for a frame rate per second (FPS) of 25. Out of this clip we keep a number of frames specified in the main paper. By default, we sample two different clips to form positives and we apply the \emph{strong-$\alpha$} and \emph{strong-$\beta$} augmentations, defined in Tab. 1 in the main paper, to the views.

\textbf{Linear evaluation protocol details.} We follow \citet{Feichtenhofer2021} and train a linear classifier for 60 epochs on top of the frozen encoder with a batch size of 512. We use the SGD optimizer with a momentum of $0.9$ and without weight decay to reach the initial learning rate $2$ that follows the linear scaling rule with the batch size set to 512. A linear warmup is applied during 35 epochs and then a cosine annealing scheduler. For training, we sample randomly a clip in the video and random crop to the size $224 \times 224$ after short scaling the video to $256$. An horizontal flip is also applied with a probability of $0.5$. For evaluation, we follow the standard evaluation protocol of \citet{Feichtenhofer2019} and sample 10 temporal clips with 3 different spatial crops of size $256 \times 256$ applied to each temporal clip to cover the whole video. The final prediction is the mean average of the predictions of the 30 clips sampled.

\textbf{Finetuning evaluation protocol details.} We follow \citet{Feichtenhofer2021} for finetuning on UCF101 and HMDB51. We finetune the whole pretrained network and perform supervised training on the 101 and 51 classes respectively for 200 epochs with dropout of probability 0.8 before classification. We use the SGD optimizer with a momentum of $0.9$ and without weight decay to reach the initial learning rate $0.1$ that follows the linear scaling rule with the batch size set to 64 and a cosine annealing scheduler without warmup. For training, we sample randomly a clip in the video and random crop to the size $224 \times 224$ after short scaling the video to $256$. We apply color jittering with the \emph{strong} augmentation parameters, defined in Tab. 1 in the main paper, and an horizontal flip with a probability of $0.5$. For evaluation, we follow the 30-crops procedure as for linear evaluation. Specific hyperparameter search for each dataset might improve results.

\section{Temperature influence on small and medium datasets}\label{sec-temp}

We made a temperature search on CIFAR10, CIFAR100, STL10 and Tiny-ImageNet by varying $\tau$ in $\{0.1, 0.2\}$ and $\tau_m$ in $\{0.03, ..., 0.10\}$. The results are in \cref{tab:temperature-small}. As for ImageNet100, we need a sharper distribution on the output of the momentum encoder. Unlike ReSSL \citep{Zheng2021}, SCE do not collapse when $\tau_m \rightarrow \tau$ thanks to the contrastive aspect. For our baselines comparison in Sec. 4.2, we use the best temperatures found for each dataset.

\end{appendix}

\end{document}